\documentclass[11pt,letterpaper]{article}
\usepackage[margin=1in]{geometry}
\usepackage[authoryear,numbers,semicolon,square]{natbib}

\usepackage{hyperref}
\usepackage{url}
\usepackage{booktabs}
\usepackage{amsfonts}
\usepackage{nicefrac}
\usepackage{microtype}
\usepackage[ruled,vlined,linesnumbered]{algorithm2e}
\usepackage{xcolor}

\usepackage{amsmath,amsthm,amssymb}
\usepackage{verbatim}
\usepackage{bm}
\usepackage[inline]{enumitem}

\hypersetup{
  colorlinks   = true,
  urlcolor     = blue,
  linkcolor    = blue,
  citecolor   = blue
}
\setlist[enumerate, 1]{label=(\alph*)}
\setlist[enumerate]{nosep}

\theoremstyle{plain}
\newtheorem{thm}{Theorem}
\newtheorem{prop}{Proposition}
\newtheorem{lem}{Lemma}[section]

\theoremstyle{definition}
\newtheorem{defn}{Definition}
\newtheorem{assum}{Assumption}

\newtheorem{prob}{Problem}

\theoremstyle{remark}

\newcommand{\logterm}{\log(2 / \delta)}

\newcommand{\gap}{\Delta}

\newcommand{\sset}{Z}
\newcommand{\srew}{\zeta}
\newcommand{\fset}{W}
\newcommand{\appv}{\tilde{v}}
\newcommand{\appnu}{\tilde{\nu}}

\DeclareMathOperator*{\argmax}{arg\,max}

\SetCommentSty{rmfamily}

\SetKwComment{Comment}{$\triangleright$\ }{}

\title{Fully Gap-Dependent Bounds for Multinomial Logit Bandit}

\author{Jiaqi Yang\\
Tsinghua University\\
\texttt{yangjq17@gmail.com}
}

\begin{document}

\maketitle

\begin{abstract}
  We study the multinomial logit (MNL) bandit problem, where at each time step, the seller offers an assortment of size at most $K$ from a pool of $N$ items, and
 the buyer purchases an item from the assortment according to a MNL choice
 model. The objective is to learn the model parameters and maximize the expected revenue. We present (i) an algorithm that identifies the optimal assortment $S^*$ within $\widetilde{O}(\sum_{i = 1}^N \Delta_i^{-2})$ time steps with high probability, and (ii) an algorithm that incurs $O(\sum_{i \notin S^*} K\Delta_i^{-1} \log T)$ regret in $T$ time steps. To our knowledge, our algorithms are the \emph{first} to achieve gap-dependent bounds that \emph{fully} depends on the suboptimality gaps of \emph{all} items. Our technical contributions include an algorithmic framework that relates the MNL-bandit problem to a variant of the top-$K$ arm identification problem in multi-armed bandits, a generalized epoch-based offering procedure, and a layer-based adaptive estimation procedure.
\end{abstract}
\section{Introduction}
\label{sec:intro}
The multinomial logit bandit (MNL-bandit) problem is an important problem in online revenue management and has attracted much attention from both operations research and online learning literature \citep{kok2007demand, rusmevichientong2010dynamic, saure2013optimal, agrawal2016near, agrawal2017thompson, chen2018note, agrawal2019mnl, wang2018near}. In MNL-bandit, at each time step, the seller offers an assortment of size at most $K$ from the pool of $N$ homogeneous items and the buyer purchases an item from the assortment according to the MNL choice model, which is arguably the simplest and most widely used discrete choice model \citep{train2009discrete,luce2012individual,soufiani2013preference} and has deep theoretical foundations \citep{mcfadden1973conditional}. The objective of the seller is to learn the model parameters and maximize the expected revenue through sequentially offering the assortments. MNL-bandit captures the essence of many real-world applications, such as retailing, where the retailer presents a limited number of products on the shelf and the customer purchases an item according to the choice model, and online advertising, where the ad platform displays a limited number of ads and the user clicks one ad according to the choice model. 

In this paper, we study the PAC exploration problem and the regret minimization problem in MNL-bandit, with a focus on proving \emph{fully gap-dependent} sample complexity and regret bounds that depends on the suboptimality gaps of \emph{all} items (detailed in Section \ref{sec:model}). There are strong practical motivations to study these bounds, because they adapt to every MNL-bandit instance and thus lead to better performances on good practical instances.  Unfortunately, there is a lack of studies on such bounds in previous MNL-bandit literature, and bounds in other bandits problems focusing on subset selection do not directly translate to our setting due to the limited feedback issue. We review these in Section \ref{sec:review} after introducing our challenge, results, and technical contributions. 

A central challenge in obtaining fully gap-dependent bounds for MNL-bandit is that the partial order between two items can be \emph{interfered} by other items. We recall an important reason why we have such bounds in other bandits settings is that we can obtain pairwise partial orders between arms to early decide on the optimality of some arms. Arms being decided need no longer be explored and stop contributing to the bounds. For example, in the top-$K$ arm identification problem, we obtain the partial order between two arms by comparing their estimations (which is independent of the estimations of other items) and we decide on arms by whether they could have top-$K$ means. While in our setting, the comparison of two items can be interfered by other items through changing their weights in the revenue function. We underscore that the changing of weights is done by changing the denominator of the \emph{fractional} revenue function, which contains model parameters of \emph{other} items. 

\paragraph{Results} We define the gaps for items and the problems in Section \ref{sec:model}. Our definitions match the intuition and naturally extend the definitions in other bandits settings. 

Our main results are three MNL-bandit algorithms with fully gap-dependent guarantees. For the PAC exploration problem, we present a $\delta$-PAC algorithm with sample complexity  $\widetilde{O}(\sum_{i = 1}^N \gap_i^{-2})$ and a $(\delta, \varepsilon)$-PAC algorithm with similar guarantee. For the regret minimization problem, we present an algorithm with $O(\sum_{i \in [N] \setminus S^*} \frac{K \log T}{\gap_i})$ regret bound.

When $K = 1$, our MNL-bandit setting becomes the multi-armed bandit setting and our bounds recover their instance-optimal sample complexity and regret bounds \citep{lai1985asymptotically, auer2002finite, slivkins2019introduction, lattimore2020}. When $K \ge 2$, our regret bound recovers the $O(\frac{N^2 \log NT}{\Delta})$ global gap-dependent regret bound in \citep{agrawal2019mnl}, because by definition we have $\gap_i \ge \Delta$, where $\Delta = \theta^* - \max_{S \subseteq [N]: \lvert S \rvert \le K}\{R(S, \bm v) : R(S, \bm v) \ne \theta^*\}$ is the gap between the optimal and second-best assortments. We compare our sample complexity bound with a previous gap-independent bound in Section \ref{sec:peresult} after presenting the theorems. 

\paragraph{Technical Contributions} We present our three techniques under the context of the PAC exploration problem in Section \ref{sec:pe} and we extend them to the regret minimization problem in Section \ref{sec:sarreg}. 

Our first technique is an algorithmic framework in Section \ref{sec:reduction}, which resolves our central challenge by the relation in Proposition \ref{prop:topk}. The relation suggests we obtain the pairwise partial order of two items by comparing the confidence intervals of their advantage scores (Definition \ref{defn:advscore}) and early decide the items according to whether they could have positive and top-$K$ advantage scores. Since this early decision rule is similar to that of the top-$K$ arm identification problem, we modify the successive accept-reject algorithm for the latter problem to obtain an algorithm with fully gap-dependent guarantees. We add a caveat that the framework itself does \emph{not} conclude, because estimating the advantage score is not a trivial job. As we show in Lemma \ref{lem:naiveguarantee}, a naive estimation procedure using only methods in previous work could lead to two extra $K$ factors in the sample complexity bound.

Our second technique eliminates an extra $K$ factor by removing some dependency in the naive procedure, as we present in Section \ref{sec:reduced}. An anomaly in the naive procedure is that we need to explore accepted items even though we do not need their scores. To remove this dependency on accepted items, we define a reduced revenue function  that requires estimating a \emph{ratio} of the model parameters. However, previous work only showed how to estimate the parameters themselves using the epoch-based offering procedure \citep{agrawal2016near, agrawal2019mnl}, with which we have to separately estimate the numerator and denominator of the ratio and suffer a huge error in the estimation. To resolve this, we generalize the epoch-based offering  procedure to \emph{directly} estimate the ratio. 

Our third technique eliminates another extra $K$ factor by a layer-based adaptive estimation procedure in Section \ref{sec:adaptive}. By carefully examining the error sources in the estimations of advantage scores, we find that the number of exploration for each item  should adapt to the error it incurs, but exact adaption requires full knowledge of the unknown model parameters. So we surrogate by putting items with similar adaption requirements into the same layer and handling them altogether. We emphasize that the layers are still \emph{unknown} and they could \emph{vary} from phase to phase. We highlight that our surrogate method re-estimates the layers for \emph{each} phase while only pays the sample complexity cost \emph{once}. 

We remark that our first technique  indeed provides a systematic way to apply fractional programming (the method that proves our relation proposition) to online learning settings. Thus it may be of independent interests.  Our second and third techniques utilizes the delicate structure of the MNL model, which could inspire future studies on MNL-bandit and other bandits with MNL model. 

\subsection{Related Work}

\label{sec:review}

MNL-bandit was first studied in \citep{rusmevichientong2010dynamic, saure2013optimal}, where the algorithms required the knowledge of the global suboptimality gap $\Delta$ in advance. Upper confidence bound-type algorithm and Thompson sampling were shown to achieve an $\widetilde{O}(\sqrt{NT})$ minimax regret bound \citep{agrawal2016near, agrawal2017thompson}. A matching $\Omega(\sqrt{NT})$ regret lower bound was shown in \citep{chen2018note}. The first gap-dependent $O(\frac{N^2 \log NT}{\Delta})$ regret bound was shown in \citep{agrawal2019mnl}. All bounds we mentioned are regret bounds, since no previous literature discussed the PAC exploration problem. Although there was a reduction from the MNL-bandit to the multi-armed bandit \citep{agrawal2016near, agrawal2019mnl}, that reduction involves exponentially many arms and thus does not give good gap-dependent bounds.

There is a line of work in  multi-armed bandits and combinatorial multi-armed bandits that studies the subset selection problem, where an algorithm learns a subset to maximize a reward function. Near-optimal fully gap-dependent regret and sample complexity bounds have been proved in those settings \citep{bubeck2012regret, chen2017adaptive, chen2016generalreward, chen2013combinatorial, chen2014combinatorial,  chen2016combinatorial, rejwan2020top}. While MNL-bandit can be seen as a subset selection problem, the major difference is that the feedback in our setting is much more \emph{limited}. In their settings, by selecting a subset (some called ``super arm''), the player gets feedback from all arms in the subset. In our setting, the seller can  obtain feedback  only from the purchased item in the subset. 

Some recent paper studies the MNL choice model under the dueling bandits framework \citep{chen2018nearly, saha2019combinatorial}, proving fully gap-dependent bounds.  Their setting can be seen as a simplification of ours through assuming all items have the same reward $r_i\equiv 1$ and removing the ``no purchase'' decision. In their setting, the optimal assortment simply consists of the items with largest model parameters, so their focus is to learn the order of the parameter. In our setting, the optimal assortment depends on the model parameters in a more complicated manner, so we need to learn the parameters themselves.

\section{Preliminaries}
\label{sec:model}

\paragraph{Notations} We define $\alpha \land \beta = \min\{\alpha, \beta\}, \alpha \lor \beta = \max\{\alpha, \beta\}$. For any two expressions $\alpha$ and $\beta$, if there exists a constant $C > 0$ in digits such that $\alpha \le C \cdot \beta$, we write $\alpha \lesssim \beta$. If $\beta \lesssim \alpha$, we write $\alpha \gtrsim \beta$. If $\alpha \lesssim \beta$ and $\beta \lesssim \alpha$, we write $\alpha \asymp \beta$.  The notions $\widetilde{O}$ and $\widetilde{\Theta}$ suppress the logarithmic terms  and the relatively small gap-independent terms in sample complexity bounds, and the logarithmic terms in regret bounds. We use both $\lvert A \rvert$ and $\#A$ to denote the size of a set $A$.   For two disjoint sets $A, B$ that $A \cap B = \emptyset$, we use $A \sqcup B = A \cup B$ to denote their union. 

\paragraph{Settings and Problems} We define the MNL choice model with parameter $v_i$ for each item $i \in [N] \cup \{0\}$, where item $i = 0$ stands for the ``no purchase'' decision. In this model, when the seller offers an assortment $S$, the buyer purchases item $i \in S \cup \{0\}$ with probability $P_{S}^{\bm v}(i) = \frac{v_i}{v_0 + \sum_{j \in S} v_j}$. Note that ``no purchase'' decision is always available to the buyer. 

We define an MNL-bandit instance as a quadruple $\mathcal I = (N, K, \bm r, \bm v)$, where the reward of item $i \in [N]$ is $r_i \in [0, 1]$ and its MNL model parameter is $v_i \in [0, 1]$. The seller knows $N, K, \bm r$, but does not know $\bm v$. At each time step $t = 1, 2, \ldots, $ the seller offers an assortment $S_t \subseteq [N]$ under the capacity constraint $\lvert S_t \rvert \le K$ and receive the buyer's purchase decision $c_t \sim P_{S_t}^{\bm v}$. As a result, the seller's revenue is $R(S, \bm v) = \mathbb E_{i \sim P_S^{\bm v}}[r_i] = \frac{\sum_{i \in S} v_i r_i}{1 + \sum_{i \in S} v_i}$, where we assume that $r_0 = 0$. We adopt a common convention that $v_0 = 1$, which means the ``no purchase'' decision is the most frequent outcome \citep{agrawal2016near, agrawal2017thompson, agrawal2019mnl}. We use $S^* = \argmax_{S \subseteq [N] : \lvert S \rvert \le K} R(S^*, \bm v)$ to denote the optimal assortment and $\theta^* = R(S^*, \bm v)$ to denote its revenue. 
Next we formally define the suboptimality gap for each item.

\label{sec:gap}

\begin{defn}[Suboptimality gap] \label{defn:sub} For every item $i \in [N]$, we define its suboptimality gap as
\begin{align*}\gap_i = \begin{cases}
R(S^*, \bm v) - \max\limits_{\lvert S \rvert \le K : i \in S} R(S, \bm v), & i \notin S^* \\
R(S^*, \bm v) - \max\limits_{\lvert S \rvert \le K : i \notin S} R(S, \bm v), & i \in S^*.
\end{cases}
\end{align*}
\end{defn}

Our definition  has the same form as the suboptimality gaps in other bandits problems focusing on subset selection \citep{bubeck2013multiple, chen2014combinatorial}. Note that the bounds usually inversely depends on the gaps, so our definition  matches the intuition that  items with small $\Delta_i$ are more difficult to be separated from the optimal assortment and thus lead to worse bounds. We make the following uniqueness assumption, which is typically assumed when studying gap-dependent bounds in bandits literature \citep{bubeck2013multiple, chen2017adaptive, karnin2013almost}.  

Our definition is related to the global gap $\Delta$ studied in previous literature  \citep{rusmevichientong2010dynamic, saure2013optimal, agrawal2019mnl}, by that we have $\gap_i \ge \Delta$ for every item $i$. We mention again that $\Delta = \theta^* - \max_{S \subseteq [N]: \lvert S \rvert \le K}\{R(S, \bm v) : R(S, \bm v) \ne \theta^*\}$ is the gap between the optimal and second-best assortments.

\begin{assum}[Uniqueness] \label{assum:uniqueopt} The optimal assortment $S^*$ is unique. 
\end{assum}

Finally, we defining the two problems we study. The first problem is defined in light of the PAC (probably approximately correct) learning framework and follows the definitions of the exploration problems in other bandits  under the fixed-confidence setting \citep{jamieson2014best, rejwan2020top}. The second problem follows the regret definition in previous MNL-bandit literature \citep{agrawal2016near, agrawal2017thompson, agrawal2019mnl, chen2018nearly, chen2018note}.

\begin{prob}[PAC Exploration] \label{prob:pe} An algorithm is $(\delta, \varepsilon)$-PAC with sample complexity $T$, if it returns an assortment $S$ that $\theta^* - R(S, \bm v) \le \varepsilon$ in $T$ time steps with probability $1 - \delta$. If $\varepsilon = 0$, we say it is $\delta$-PAC. The goal is to design $\delta$-PAC and $(\delta, \varepsilon)$-PAC algorithms with minimum sample complexity.
\end{prob}

\begin{prob}[Regret Minimization] \label{prob:reg} The goal is to design an algorithm that offers assortments over a known time horizon $T (\ge N)$ with minimum regret $\mathrm{Reg}_T = \sum_{t = 1}^T R(S^*, \bm v) - \mathbb E[R(S_t, \bm v)]$. 
\end{prob}

\section{PAC Exploration}
\label{sec:pe}

\subsection{Algorithmic Framework with Fully Gap-Dependent Bounds}

\label{sec:reduction}

In this subsection, we introduce an algorithmic framework for which we can obtain fully gap-dependent sample complexity bounds and, as a direct application, present a $\delta$-PAC algorithm with sample complexity $\widetilde{O}(\sum_{i = 1}^N K^2 \gap_i^{-2})$. Our framework is based on relating the MNL-bandit problem to the positive top-$K$ item identification (PTOP-$K$) problem via the notion of advantage score. 
\paragraph{Relate MNL-bandit to PTOP-$K$} We first describe the goal of the PTOP-$K$ problem, then relate it with the MNL-bandit problem. To describe the goal, we define the following function $\mathcal F$. Given a capacity constraint $M$ and a set $\fset$ where each $i \in \fset$ has a score $\xi_i \in \mathbb R$, we denote the subset containing elements with positive and top-$M$ scores as 
\begin{align}
    \mathcal F(\fset, M, \xi) = \{i \in \fset : \xi_i > 0\} \cap \{i \in \fset : \xi_i \text{ is among the top }M\text{ of }\{\xi_j\}_{j \in \fset} \}. 
\end{align}
The goal of the PTOP-$K$ problem is to identify the subset $\mathcal F([N], K, \bm u)$ of items, where $u_i$ is the specially constructed score defined with respect to each item $i \in [N]$ as follows.
\begin{defn}[Advantage Score] \label{defn:advscore} We define the \emph{advantage score} of item $i$ as $u_i = v_i(r_i - \theta^*)$. 
\end{defn}
Now we relate the MNL-bandit problem to the PTOP-$K$ problem by the following proposition, which states that they share the \emph{same} goal of identifying the optimal assortment $S^* \subseteq [N]$. 

\begin{prop}[Relate to PTOP-$K$] \label{prop:topk} $S^* = \mathcal F([N], K, \bm u)$ and $\theta^* = \sum_{i \in S^*} u_i$.
\end{prop}
We defer the proof to Appendix \ref{app:prooftopk}, which uses a classical method in optimization theory called fractional programming \citep{dinkelbach1967nonlinear, rusmevichientong2010dynamic}. Our proposition indicates that pairwise partial orders and  early decision rules in MNL-bandit are the same as those in the PTOP-$K$ problem, which is very similar to the top-$K$ arm identification problem. Since algorithms with fully gap-dependent bounds are well-studied in the top-$K$ arm problem, we can obtain such bounds for the MNL-bandit problem by combining those algorithms with our relation proposition. 
 
 However, two issues arise when combining them. First, the gap-dependent bounds for the top-$K$ problems use the gaps of scores, not our suboptimality gap for items. Second, estimating the advantage scores $u_i$ is much more difficult than estimating the means of arms in the top-$K$ arm problem, because the definition of $u_i$ involves the optimal revenue $\theta^*$, which could depend on items \emph{other} than $i$. In contrast, the mean of each arm only depends on the arm itself. The first issue can be resolved by Lemma \ref{lem:diui3}, which shows that our gap is always smaller and thus bounds for top-$K$ problems translate to our MNL-bandit setting. The second issue is difficult to resolve. In Lemma \ref{lem:naiveguarantee}, we will show that a naive solution could lead to two extra $K$ factors in the guarantee. 

\paragraph{Algorithmic Framework} Let us assume a procedure \textsf{EST} that estimates the advantage score. We introduce our algorithmic framework \textsf{SAR-MNL} (Algorithm \ref{algo:peframework}). We summarize below its sample complexity guarantee and defer the proof to Appendix \ref{app:proofpeframework}.

\begin{lem}[\textsf{SAR-MNL}] \label{lem:peframework} Assume with probability  $1 - \delta^{(k)}$, \textsf{EST} (a) returns within $C_{\textsf{EST}} \cdot \frac{ \lvert B^{(k-1)} \rvert \log(N / \delta^{(k)})}{\epsilon_k^2}$ time steps in phase $k$, and (b) $u_i \in [\check \xi_i, \hat \xi_i]$ and $\hat \xi_i - \check \xi_i \le \frac{\epsilon_k}{2}$ for every $i \in B^{(k-1)}$. Then \textsf{SAR-MNL} with \textsf{EST} is $\delta$-PAC with sample complexity $C_{\textsf{EST}} \cdot O(\sum_{i \in [N]} \frac{\log N + \log \delta^{-1} + \log \log \gap_i^{-1}}{\Delta_i^2})$.
\end{lem}

 Our framework is similar to the successive accept-reject algorithms used to solve the top-$K$ arm identification problem \citep{chen2017adaptive, rejwan2020top, bubeck2013multiple}. 
The idea is to alternate in phases between estimate the scores of pending items and accept-reject them. For each phase $k$, accepted items are stored in $A^{(k)}$ and rejected items are in $[N] \setminus (A^{(k)} \cup B^{(k)}).$ In phase $k$, after building up the confidence intervals of scores $u_i$ at Line \ref{loc:callest}, the algorithm accepts-rejects items by some rules. Since
\begin{align*}
    S^* = \mathcal F([N], K, \bm u) = \mathcal F(A^{(k-1)} \cup B^{(k-1)}, K, \bm u) 
    = A^{(k-1)} \sqcup \mathcal F(B^{(k-1)}, M, \bm u),
\end{align*}
where $M = \min\{K - \lvert A^{(k-1)} \rvert, \lvert B^{(k-1)} \rvert\}$,  the rules are to accept items in $B^{(k-1)}$ with positive and top-$M$ scores and reject those with negative or not top-$M$ scores. In the framework, Line \ref{loc:arpositive} handles the sign rule and Lines \ref{loc:artopm}-\ref{loc:artopmend} handle the top-$M$ rule. 

\begin{algorithm}[t]
  \label{algo:peframework}
  \caption{\textsf{SAR-MNL}($\delta$): Successive Accept-Reject Framework for MNL-bandit} 
  $A^{(0)} = \emptyset, B^{(0)} = [N]$\Comment*[r]{ $A^{(k)}, B^{(k)}$ store accepted, pending items}
  \For(\Comment*[f]{ maintain $A^{(k)} \subseteq S^* \subseteq A^{(k)} \sqcup B^{(k)}$ for each phase $k$}){$k \gets 1, 2, \ldots$}{ 
    $\epsilon_k = 2^{-k}, \delta^{(k)} = \frac{\delta}{3 k^2}, M = M^{(k-1)} = \min\{K - A^{(k-1)}, \lvert B^{(k-1)} \rvert\}$\; 
    \lIf{$M = 0$}{\Return{$A^{(k-1)}$}}
    $\{\check \xi_i, \hat \xi_i\}_{i \in B^{(k-1)}} \gets \mathsf{EST}(A^{(k-1)}, B^{(k-1)}, \delta^{(k)}, \frac{\epsilon_k}{2})$\Comment*[r]{ estimate the scores } \label{loc:callest}
    $B_{\mathrm{acc}} \gets \{b \in B^{(k-1)} : \check \xi_b > 0 \}, B_{\mathrm{rej}} \gets \{b \in B^{(k-1)} : \hat \xi_b < 0\}$\; \label{loc:arpositive} \label{loc:arstart}
    \If(\Comment*[f]{ if $\lvert B^{(k-1)} \rvert \le M$ then all items have top-$M$ scores }){$\lvert B^{(k-1)} \rvert > M$}{ \label{loc:artopm}
    $\alpha \gets M$-th largest value of $\{\check \xi_i\}_{i \in B^{(k-1)}}, \beta \gets (M+1)$-th largest value of $\{\hat \xi_i\}_{i \in B^{(k-1)}}$\;
    $B_{\mathrm{acc}} \gets B_{\mathrm{acc}} \cap \{b \in B^{(k-1)} : \check \xi_b > \beta \}, B_{\mathrm{rej}} \gets B_{\mathrm{rej}} \cup \{b \in B^{(k-1)} : \hat \xi_b < \alpha \}$\; \label{loc:artopmend}
    }
    $A^{(k)} \gets A^{(k-1)} \cup B_{\mathrm{acc}}, B^{(k)} \gets B^{(k-1)} \setminus (B_{\mathrm{acc}} \cup B_{\mathrm{rej}})$\Comment*[r]{ accepts-rejects the items } \label{loc:ar} \label{loc:arend}
  }
\end{algorithm}

\begin{algorithm}[t]
\label{algo:naiveest}
  \caption{$\textsf{EST-NAIVE}(A, B, \delta, \epsilon)$: Naive Estimation of  $u_i$ for $i \in B$} 
  \lFor(\Comment*[f]{$\tau = \widetilde{O}(\frac{1}{\epsilon^2})$}){$i \in A \sqcup B$}{\label{loc:singletonoffernaive}Keep offering $\{i\}$ until ``no purchase'' occurs for $K \tau$ times} 
  $\forall i \in A \sqcup B$: Compute the confidence intervals $v_i \in [\check v_i, \hat v_i]$\Comment*[r]{ formulas of $\check v_i, \hat v_i$ in Appendix \ref{app:naiveguarantee} }
  Compute $\check \theta = \max_{S\subseteq A \sqcup B} R(S, \check v), \hat \theta = \max_{S\subseteq A \sqcup B} R(S, \hat v)$\; \label{loc:naiveesttheta}
  Compute $\check{\xi}_i = (\check{v}_i (r_i - \hat \theta)) \land (\hat{v}_i (r_i - \hat \theta)),  \hat{\xi}_i = (\check{v}_i (r_i - \check \theta)) \lor (\hat{v}_i (r_i - \check \theta))$, \Return{$\{ \check \xi_i, \hat \xi_i\}$}\; \label{loc:naiveestscore}
\end{algorithm}

\paragraph{Estimation Procedure} We present a naive estimation procedure \textsf{EST-NAIVE} (Algorithm \ref{algo:naiveest}). The procedure estimates the score $u_i = v_i (r_i - \theta^*)$ by estimating both $v_i$ and $\theta^*$. Line \ref{loc:naiveesttheta} is because the optimal revenue is a monotonic function of the model parameters \citep{agrawal2016near, agrawal2019mnl}. (We emphasize that the revenue is not monotonic in general.) The maximization step at Line \ref{loc:naiveesttheta} can be solve efficiently \citep{rusmevichientong2010dynamic}. Line \ref{loc:naiveestscore} is based on $0 \le v_i\le 1$ and $\lvert r_i - \theta^* \rvert \le 1$.  The procedure leads to the following guarantee. 

\begin{lem} \label{lem:naiveguarantee} \textsf{SAR-MNL} with \textsf{EST-NAIVE} is $\delta$-PAC with sample complexity $\widetilde{O}(\sum_{i \in [N]} \frac{K^2}{\Delta_i^2})$.
\end{lem}

We sketch the proof here and complete it in Appendix \ref{app:naiveguarantee}. Note that the procedure offers each item in the set $A \cup B$ for $\widetilde{O}(\frac{K}{\epsilon^2})$ time steps, so it achieves $C_{\textsf{EST}} \asymp K \cdot \frac{\lvert A \cup B \rvert}{\lvert B \rvert}$ for Lemma \ref{lem:peframework}. In the worst case, we have $ \lvert A \cup B \rvert \asymp K \lvert B \rvert$, so we have $C_{\textsf{EST}} \le K^2$, which implies Lemma \ref{lem:naiveguarantee}. 

We inspect the sources of two $K$ factors in Lemma \ref{lem:naiveguarantee}. The first is because we use that $\frac{\lvert A \cup B \rvert}{\lvert B \rvert} \le K$, which is ultimately because the naive procedure needs to estimate $v_i$ for $i \in A$. The second  is because the procedure needs to estimate each $v_i$ to a fixed accuracy $\frac{\epsilon}{K}$ in order to estimate $\theta^*$. One may ask why the procedure only offers singletons at Line \ref{loc:singletonoffernaive} and why the accuracy needs to be $\frac{\epsilon}{K}$ instead of $O(\epsilon)$. Interestingly, we show in Appendix \ref{app:naiveguarantee} that  both could be optimal for some instance.

\subsection{Reduced Revenue Function and Generalized Epoch-based Offering}

\label{sec:reduced}

To eliminate the first $K$ factor in \textsf{EST-NAIVE}, we introduce a reduced revenue function and a generalized epoch-based offering procedure in this subsection. The reduced revenue function enables us to estimate $\theta^*$ without estimating the parameters $v_i$ for $i \in A$. The generalized procedure is used to estimate the parameters in the reduced revenue function.

\paragraph{Reduced Revenue Function} We note that \textsf{SAR-MNL} invokes \textsf{EST} with $A \subseteq S^* \subseteq A \sqcup B$. However, \textsf{EST-NAIVE} only uses $S^* \subseteq A \sqcup B$. Now we exploit $A \subseteq S^*$. Let $M = \min\{K - \lvert A \rvert, \lvert B \rvert\}$. For an assortment $S$ satisfying $A \subseteq S$, we rewrite its revenue as 
\begin{align*}
    R(S, \bm v) = \frac{\sum_{i \in S} v_i r_i}{1 + \sum_{i \in S} v_i}  = \frac{\srew + \sum_{i \in S \setminus A} \nu_i r_i}{1 + \sum_{i \in S \setminus A} \nu_i} = R(S\setminus A, \nu, \srew),
\end{align*}
where we define $\srew =  \frac{\sum_{j \in A} v_j r_j}{1 + \sum_{j \in A} v_j} = R(A, \bm v)$ and $\nu_i = \frac{v_i}{1 + \sum_{j \in A} v_j}$ for $i \notin A$. Note that if we use $R(S, \bm v)$ to compute the revenue of $S$, we need $\lvert S \rvert$ parameters ($v_i$ for each $i \in S$). In contrast, if we use $R(S \setminus A, \nu, \srew)$, we only need $(\lvert S \setminus A \rvert + 1)$ parameters ($\nu_i$ for $i \in S \setminus A$ and $\srew$). Thus we refer to the function $R(S \setminus A, \nu, \srew)$ as the \emph{reduced revenue function}, since it reduces the number of required parameters. We note that $\theta^* = \max_{S_0 \subseteq B: \lvert S_0 \rvert \le M} R(S_0, \nu, \srew)$ and Lemma \ref{lem:mono} further shows that the maximization used by $\theta^*$ is still monotonic in the parameters $\nu$ and $\srew$. Therefore, given the confidence intervals $\srew \in [\check \srew, \hat \srew]$ and $\nu_i \in [\check \nu_i, \hat \nu_i]$, we have the confidence interval $\theta^* \in [\check \theta, \hat \theta]$, where
\begin{align}
    \check{\theta} = \max_{S\subseteq B: \lvert S \rvert \le M} R(S, \check{\nu}, \check{\srew}), \qquad 
    \hat{\theta} = \max_{S \subseteq B: \lvert S \rvert \le M} R(S, \hat{\nu}, \hat{\srew}). \label{eq:esttheta}
\end{align}

\begin{algorithm}[t]
  \caption{$\textsf{Explore}(S)$: Generalized Epoch-based Offering with Stopping Set $\sset (\sset \cap S = \emptyset)$} \label{algo:explore}
  Initialize: $z \gets 0, \ell \gets \ell + 1, E_\ell = 0, \forall i \in S: x_i \gets 0$\;
  \While(\Comment*[f]{Epoch: time steps used in the while-loop}){\textsc{true}}{
    $t \gets t + 1, E_{\ell} \gets E_{\ell} + 1$\Comment*[r]{ $E_\ell$ is the length of epoch $\ell$ }
    Offer assortment $S_t = \sset \cup S$, observe purchase decision $c_t$\;
    \leIf{$c_t \in \sset \cup \{0\}$}{$z \gets r_{c_t}$, break;}{$x_{c_t} \gets x_{c_t} + 1$}
  }
  $n_\sset \gets n_\sset + z, T_\sset \gets T_\sset + 1, \bar{\srew} \gets \frac{n_\sset}{T_\sset}$, $\forall i\in S: n_i \gets n_i + x_i, T_i \gets T_i + 1, \bar{\nu}_i \gets \frac{n_i}{T_i}$\;
\end{algorithm}

\paragraph{Generalized Epoch-based Offering} With Eq. (\ref{eq:esttheta}) in hand, it remains how to estimate $\srew$ and $\nu_i$. Note that $\nu_i$ is a \emph{ratio} of two \emph{unknown} quantities $v_i$ and $(1 + \sum_{j \in A} v_j)$, so it is virtually impossible to estimate $\nu_i$ by separately estimating the two quantities. The generalized epoch-based offering procedure (Algorithm \ref{algo:explore}) allows us to \emph{directly} estimate the ratio $\nu_i$. It generalizes those used in \citep{agrawal2016near,agrawal2017thompson, agrawal2019mnl} by introducing a stopping set $\sset$, which is fixed as $\sset = \emptyset$ in the original version. When we set $\sset = A$, we can use the procedure to estimate parameters $\nu_i$ and also $\srew$.

\begin{prop}[Generalized Epoch-based Offering] \label{prop:ind} After $\textsf{Explore}(S)$, we have 
\begin{enumerate}
    \item $z \in [0, 1]$ is an independent bounded random variable with mean $\srew$;
    \item $x_i$ is an independent geometric random variable with mean $\nu_i$ for every item $i \in S$;
    \item $(E_\ell - 1)$ is an independent geometric random variable with mean $\sum_{i \in S} \nu_i$.
\end{enumerate}
\end{prop}
We defer the proof to Appendix \ref{app:propind}. Statement (c) can give the sample complexity bound when using the procedure, as in Lemma \ref{lem:sumofepochlength}.  Combined with corresponding concentration inequalities in Appendix \ref{app:ci}, statements (a)(b) can give the confidence intervals $\srew \in [\check \srew, \hat \srew]$ and $\nu_i \in [\check \nu_i, \hat \nu_i]$, where
\begin{align}
\check{\srew} = 0 \lor (\bar \srew - \sqrt{\frac{\log(2 / \delta)}{2 T_\sset}}),&\qquad \hat{\srew} = 1 \land (\bar \srew + \sqrt{\frac{\log(2 / \delta)}{2 T_\sset}}), 
\label{eq:estz} \\
    \check{\nu}_i = 0 \lor (\bar{\nu}_i - \sigma(\nu_i)), \quad \hat{\nu}_i = 1 \land (\bar{\nu}_i + \sigma(\nu_i)),&\qquad    \sigma(\nu_i) = \sqrt{\frac{48 \bar{\nu}_i \logterm}{T_i}} +  \frac{48 \logterm}{T_i}.
\label{eq:estv}
\end{align}
Finally, we define the reduced score $\xi_i = \nu_i (r_i - \theta^*)$ and its confidence interval $\xi_i \in [\check \xi_i, \hat \xi_i]$, where
\begin{align}
    \check{\xi}_i = (\check{\nu}_i (r_i - \hat \theta)) \land (\hat{\nu}_i (r_i - \hat \theta)), \qquad
    \hat{\xi}_i = (\check{\nu}_i (r_i - \check \theta)) \lor (\hat{\nu}_i (r_i - \check \theta)).\label{eq:estu}
\end{align}
Now we assume the procedure \textsf{EST} used by \textsf{SAR-MNL} estimates $\xi_i$ instead of $u_i$. In Appendix \ref{app:enhancedpeframework}, we show that the same sample complexity bound as Lemma \ref{lem:peframework} still hold. 

To demonstrate the technique in this subsection, we show in Appendix \ref{app:reducedguarantee} that we can achieve $C_{\textsf{EST}} = O(K)$ using the generalized epoch-based offering, which implies an $\widetilde{O}(\sum_{i = 1}^N K \gap_i^{-2})$ sample complexity bound and eliminates an extra $K$ factor in \textsf{EST-NAIVE}.

\subsection{Layer-based Adaptive Estimation}

\label{sec:adaptive}

To eliminate another extra $K$ factor in $\textsf{EST-NAIVE}$, we present a layer-based adaptive estimation procedure based on a detailed error analysis of the reduced revenue function and the tail bounds in Eqs. (\ref{eq:estz})(\ref{eq:estv}). The error analysis  in Appendix \ref{app:erralyz} suggest we offer each item $b \in B$  for $T_b \gtrsim T'_b \tau$ epochs, where we define $T'_b = (\frac{1}{\nu_b} \land M)$ and $\tau = \widetilde{O}(\epsilon^{-2})$. Next we show how to accomplish this offering task in  $O(\lvert B \vert \tau)$ time steps, which gives $C_{\textsf{EST}} = O(1)$ in Lemma \ref{lem:peframework} and eliminates the extra $K$ factor. To better convey our idea, we first consider an ideal but unrealistic case where the exact values of $\nu_b$ are given. We divide the set $B$ into $m = \lceil \log_2 M \rceil$ layers:
\begin{align}
 B_i = \{b \in B : \nu_b \in (2^{-(i+1)}, 2^{-i}]\}\quad \text{for } 0\le i < m,\qquad
 B_m = \{b \in B : \nu_b \in [0, 2^{-i}]\}.
 \label{eq:coarse}
\end{align}
 Let $d_i = 2^i$ for $i < m$ and $d_m = M$. The key observation is that items form the same layer have similar $\nu_b$ and need to be explored for a similar number of epochs (up to a factor $\kappa = 2$): we have $T'_b \le \kappa d_i$ and $\nu_b \le \frac{\kappa}{d_i}$ for $b \in B_i$. We note that $d_i \le M$. Therefore, we can divide each layer $B_i$ into groups of size $d_i$. Since we have $\nu_b \lesssim \frac{1}{d_i}$ for $b \in B_i$, by Proposition \ref{prop:ind}, the expected epoch length of explore a group is $d_i \cdot \frac{1}{d_i} = O(1)$. So if we explore each group for $d_i \tau$ epochs, in expectation it costs us $d_i \tau$ time steps, which is  $\tau$ time steps per item. Since we have $\lvert B \rvert$ items, we can accomplish the offering task within $O(\lvert B \rvert \tau)$ time steps, which gives the desired $C_{\textsf{EST}} = O(1)$.

\begin{algorithm}[t]
\label{algo:coarse}
\caption{$\textsf{EST-ROUGH}(\delta_0)$: Rough Estimation of $v_i$ for $i \in [N]$}
$C_0 = 196, \delta = \frac{\delta_0}{17 N}, \tau = 4 K C_0 \logterm, Z \gets \emptyset, \forall i \in [N]: n_i = T_i = 0$\;
\lFor{$i \in [N]$}{$\mathsf{Explore}(\{i\})$ for $\tau$ epochs}
$\forall i \in [N]:$ compute $\hat{\nu}_i$ by Eq. (\ref{eq:estv}), let $\appv_i \gets \hat{\nu}_i$, \Return{$\{\appv_i\}_{i \in [N]}$}\;
\end{algorithm}

\begin{algorithm}[t]
\label{algo:refined}
\caption{$\textsf{EST-ADAPTIVE}(A, B, \delta_0, \epsilon)$: Layer-based Adaptive Estimation of $\xi_i$ for $i \in B$}
$C_0 = 196, C_2 = 1024, \delta = \frac{\delta_0}{15 N}, \tau = \frac{C_2 C_0 \logterm}{\epsilon^2}, Z \gets A, n_\sset = T_\sset = 0$, $\forall i \in B: n_i = T_i = 0, \appnu_i = \frac{\appv_i}{1 + \sum_{j \in \sset} \appv_j}$\Comment*[r]{assume $\tilde{v}_i$ has been computed by \textsf{EST-ROUGH} } \label{loc:coarseestim}
    \For{$i \gets 0, 1, \ldots, m = \lceil \log_2 M \rceil$}{ 
     Compute $B_i$ by Eq. (\ref{eq:coarse})  using $\tilde{\nu}_b$ in place of $\nu_b$, let $d_i = 2^i \lor M$\; \label{loc:usecoarseestim}
       Divide $B_i = B_{i,1} \sqcup \cdots \sqcup B_{i,c_i}$ so that $\lvert B_{i, 1} \rvert = \cdots = \lvert B_{i, c_i - 1} \rvert = d_i$ and $ \lvert B_{i, c_i} \rvert \le d_i$\;
       $\forall j \in [c_i]: \textsf{Explore}(B_{i, j})$ for $d_i \tau$ epochs\;
    }
    Compute $\check{\srew}, \hat{\srew}, \check{\nu}_i, \hat{\nu}_i, \check \theta, \hat \theta, \check{\xi}_i, \hat{\xi}_i$ by Eqs. (\ref{eq:estz}) (\ref{eq:estv}) (\ref{eq:esttheta}) (\ref{eq:estu}) for $i \in B$, \Return{$\{\check \xi_i, \hat \xi_i\}_{i \in B}$}\; \label{loc:est2}
\end{algorithm}

We note that the exact values of $\nu_b$ are not necessary, since we only need the layer of each item. In fact, we can surrogate by using a \emph{rough} estimation of $\nu_b$ satisfying $\appnu_b \in [\frac{\nu_b}{2}, 2 \nu_b \lor \frac{1}{M}]$ in place of $\nu_b$ in Eq. (\ref{eq:coarse}) to divide the layers. Our key observation is still satisfied, though perhaps with a different factor $\kappa = 4$. This fact is utilized by our layer-based adaptive estimation procedure \textsf{EST-ADAPTIVE} (Algorithm \ref{algo:refined}). It first uses the results from the procedure \textsf{EST-ROUGH} (Algorithm \ref{algo:coarse}) to build up $\appnu_b$ at Line \ref{loc:coarseestim}, based on which it divides the layers at Line \ref{loc:usecoarseestim}. We summarize the guarantees of our two procedures in the following two lemmas and defer their proofs to Appendices \ref{app:pecoarse} and \ref{app:perefined}. 

\begin{lem}[\textsf{EST-ROUGH}] \label{lem:coarse} With probability $1 - \delta_0$, (a) \textsf{EST-ROUGH} ends in $O(N K \log \frac{N}{\delta_0})$ time steps and $\appv_i \in [v_i, 2 v_i \lor \frac{1}{K}]$ for every $i \in [N]$; (b) In this case, for every set $S \subseteq [N]$ with $\lvert S \rvert \le K$, we have $\appnu_i \in [\frac{\nu_i}{2}, 2 \nu_i \lor \frac{1}{K}]$, where $\nu_i = \frac{v_i}{\sum_{i \in S} v_i}$ and $\appnu_i = \frac{\appv_i}{1 + \sum_{i \in S} \appv_i}$. 
\end{lem}

\begin{lem}[\textsf{EST-ADAPTIVE}] \label{lem:refined} Assume $A \subseteq S^* \subseteq A \sqcup B$. With probability $1 - \delta_0$, (a) \textsf{EST-ADAPTIVE} returns in $\lvert B \rvert \tau$ time steps; (b) $\xi_i = \frac{u_i}{1 + \sum_{j \in A} v_j} \in [\check \xi_i, \hat \xi_i]$ and $\hat \xi_i - \check \xi_i \le \epsilon$ for $i \in B$.
\end{lem}

\subsection{Putting Everything Together}

\label{sec:peresult}
We combine all our techniques to design a $\delta$-PAC algorithm: we first invoke \textsf{EST-ROUGH}($\frac{\delta}{2}$), then invoke \textsf{SAR-MNL}($\frac{\delta}{2}$) with \textsf{EST-ADAPTIVE}. We highlight that \textsf{EST-ROUGH} is invoked only \emph{once} while it can help divide layers for \emph{every} phase, as shown by the  statement (b) in Lemma \ref{lem:coarse}. 

Our $\delta$-PAC algorithm becomes $(\delta, \varepsilon)$-PAC if we terminate it at the phase $k$ satisfying $\epsilon_k \lesssim \varepsilon$ and let it return the assortment corresponding to $\hat \theta$. We summarize the results in the below theorems.

\begin{thm} \label{thm:pe} There is a $\delta$-PAC algorithm with sample complexity
$$O(N K \log \frac{N}{\delta} + \sum_{i \in [N]} \gap_i^{-2} (\log \frac{N}{\delta} + \log \log \gap_i^{-1})).$$
\end{thm}

\begin{thm} \label{thm:pac} There is a $(\delta, \varepsilon)$-PAC algorithm with sample complexity
$$O(N K \log \frac{N}{\delta} + \sum_{i \in [N]} (\gap_i^\varepsilon)^{-2} (\log \frac{N}{\delta} + \log \log (\gap_i^\varepsilon)^{-1})),  $$
where $\Delta_i^\varepsilon = \Delta_i \lor \varepsilon$.
\end{thm}

  The proofs are deferred to Appendix \ref{app:peresult}. Both bounds have a gap-independent  $O(N K \log \frac{N}{\delta})$ term due to \textsf{EST-ROUGH}, which is arguably much smaller than the gap-dependent $\widetilde{O}(\sum_{i = 1}^N \gap_i^{-2})$ term.

Theorem \ref{thm:pac} translates to an $\widetilde{O}(N \varepsilon^{-2})$ gap-independent sample complexity bound, which matches a corollary of previous  $\widetilde \Theta(\sqrt{NT})$ minimax regret bound in  \citep{agrawal2016near, agrawal2019mnl, agrawal2017thompson, chen2018note} as follows. Suppose we run an algorithm with  $\widetilde{O}(\sqrt{NT})$ regret bound for $T = O(N \varepsilon^{-2})$ time steps and uniformly choose an assortment $S$ from $\{S_1, \ldots, S_T\}$. In expectation, we have $\mathbb E[\theta^* - R(S, \bm v)] = \widetilde{O}(\sqrt{NT}) / T = \varepsilon$ and thus we get an algorithm with $O(T) = O(N \varepsilon^{-2})$ sample complexity. 

\section{Regret Minimization}

\label{sec:sarreg}

\begin{algorithm}[t]
\label{algo:lowreg}
  \caption{\textsf{EST-REG}($A, B, \delta_0, \epsilon$): Low-Regret Estimation of $u_i$} 
  $C_0 = 196, C_2 = 1024, \delta = \frac{\delta_0}{13 N}, M = \min\{K - \lvert A \rvert, \lvert B \rvert\}, \tau = \frac{C_2 C_0   \logterm}{\epsilon_k^{2}}$, $Z = \emptyset, \forall i \in A \cup B: n_i = T_i = 0$\;
    Create $m = \left\lceil \frac{\lvert B \rvert}{M} \right\rceil$ sets $B'_{1}, \ldots, B'_{m} \subseteq B$ so that $B \subseteq B'_1 \cup \cdots \cup B'_m$ and $\lvert B'_i \rvert = M$\;
    $\forall i \in [m]$: \textsf{Explore}($A \cup B'_i$) for $K \cdot \tau$ epochs\Comment*[r]{we offer full assortments with size $(M + \lvert A \rvert)$}
    Let $\check{\srew} = \hat{\srew} = 0$, Compute $\check{\nu}_i, \hat{\nu}_i, \check \theta, \hat \theta, \check{\xi}_i, \hat{\xi}_i$ by Eqs. (\ref{eq:estv}) (\ref{eq:esttheta}) (\ref{eq:estu}) for $i \in A \cup B$, \Return{$\{\check \xi_i, \hat \xi_i\}_{i \in B}$}\; 
\end{algorithm}

In this section, we present fully gap-dependent regret bounds for Problem \ref{prob:reg}. Our algorithm is to invoke \textsf{SAR-MNL}($\frac{1}{T}$) with our low-regret estimation procedure \textsf{EST-REG} (Algorithm \ref{algo:lowreg}). This algorithm satisfies the following theorem, whose proof is deferred to Appendix \ref{app:reg}. 

\begin{thm}[Regret] \label{thm:sereg} There is an algorithm achieves regret bound $O(\sum_{i \in [N] \setminus S^*} \frac{K \log NT}{\gap_i})$.
\end{thm}

Our procedure achieves low regret by fixing the accepted set $A$ and offering full assortments. The first fixing idea has been exploited in \citep{rejwan2020top}. The second idea is our novel technique, without which we could have a regret bound $O(\sum_{i = 1}^N \frac{K \log NT}{\gap_i})$ that depends on $S^*$. 

\section{Discussion on Lower Bounds}

An interesting question is whether we can prove fully gap-dependent lower bounds in MNL-bandit for Problems \ref{prob:pe} and \ref{prob:reg}. To begin with, we prove an $\Omega(\sum_{i \notin S^*} \frac{\log T}{K \gap_i})$ regret lower bound in Appendix \ref{app:lb}, which matches our regret upper bound when $K = 1$. Besides, for Problem \ref{prob:pe}, our sample complexity upper bound matches the gap-independent sample complexity bound translated from the previous minimax-optimal regret bound, as specified earlier in Section \ref{sec:peresult}. 

Furthermore, we discuss about the difficulties in studying gap-dependent lower bounds in MNL-bandit. Given an arbitrary gap sequence $\{\gap_i\}$, it is not trivial to realize the gaps, where by realizing we mean to find an MNL-bandit instance with such gaps. The fractional revenue function made it hard to determine if a given gap sequence could correspond to an instance and construct such instance when exists. The hardness in constructing instance from gap makes it difficult to prove lower bounds by following the canonical change-one-arm lower bound argument for multi-armed bandits. We believe that perhaps for these reasons, previous work also did not prove gap-dependent lower bounds (even in the term of the global gap $\Delta$, which is much weaker than our gap definition $\Delta_i$, as discussed in Sections \ref{sec:intro} and \ref{sec:gap}) when studying gap-dependent bounds in MNL-bandit.

\section{Conclusion}

\label{sec:discuss}

In this paper, we develop multiple techniques to prove the fully gap-dependent sample complexity and regret bounds for the MNL-bandit problems. We leave it a further direction to prove tighter lower bounds for the problems. For the upper bound, a significant question is whether we can remove the $K$ factor in the regret bound. It would be worthwhile to prove a fully gap-dependent regret bound for the original upper confidence bound algorithm in  \citep{agrawal2016near, agrawal2019mnl}. An interesting direction is whether the gap-independent term in our sample complexity bound can be reduced or even totally be removed. 

\section*{Acknowledgments}

Jiaqi Yang would like to thank Yuan Zhou for the invaluable comments and suggestions.

\newpage

\appendix
\allowdisplaybreaks
\begin{center}
  \bf\LARGE{Appendices}
\end{center}

\section{Concentration Inequalities}

\label{app:ci}

We introduce the concentration inequalities used in this paper. We begin with the Hoeffding's celebrated inequality for the sum of bounded variables \citep{hoeffding1963probability}.

\begin{lem}[Chernoff-Hoeffding's inequality] \label{lem:hoeffding} Consider $n$ independent bounded random variables $X_1, \ldots, X_n \in [0, 1]$. Let $\bar X = \frac{1}{n}\sum_{i = 1}^n X_i$ and $\mu = \mathbb E[\bar X]$. We have $\mathbb P(\lvert \bar X  - \mu \rvert \ge \sqrt{\frac{\log(2 / \delta)}{2 n}}) \le \delta.$
\end{lem}

Next we state the multiplicative Chernoff inequalities to the geometric random variables \citep{agrawal2019mnl}. We say a random variable is \emph{geometric} if $\mathbb P(X = m) = p (1-p)^m$.

\begin{lem}[\citet{agrawal2019mnl}, Corollary D.1] \label{lem:mulchernoff} Consider $n$ i.i.d. geometric random variables $X_1, \ldots, X_n$ with expectation $\mathbb E[X_i] = \mu \le 1$. Let $\bar X = \frac{1}{n}\sum_{i = 1}^n X_i$. We have 
\begin{enumerate}
    \item $\mathbb P(\lvert \bar X - \mu \rvert > \sqrt{\frac{48 \bar X \log (\sqrt N \ell + 1)}{n}} + \frac{48 \log (\sqrt N \ell + 1)}{n}) \le \frac{6}{N \ell^2}$,
    \item $\mathbb P(\lvert \bar X - \mu \rvert > \sqrt{\frac{24 \mu \log (\sqrt N \ell + 1)}{n}} + \frac{48 \log (\sqrt N \ell + 1)}{n}) \le \frac{4}{N \ell^2}$,
    \item $\mathbb P(\bar X \ge \frac{3\mu}{2} + \frac{48 \log (\sqrt N \ell + 1)}{n}) \le \frac{3}{N \ell^2}$.
\end{enumerate}
\end{lem}

We rephrase the above lemma into the below form. Lemma \ref{lem:hpver1} can be proved by following Appendix D in \citep{agrawal2019mnl}. Similar inequalities with constants smaller than $48$ were shown in \citep{jin2019asymptotically, janson2018tail}.

\begin{lem} \label{lem:hpver1} Consider $n$ i.i.d. geometric random variables $X_1, \ldots, X_n$ with expectation $\mathbb E[X_i] = \mu \le 1$. Let $\bar X = \frac{1}{n}\sum_{i = 1}^n X_i$. We have 
\begin{enumerate}
    \item $\mathbb P(\lvert \bar X - \mu \rvert > \sqrt{\frac{48 \bar X \log(2 / \delta)}{n}} + \frac{48 \log(2 / \delta)}{n}) \le 6\delta$,
    \item $\mathbb P(\lvert \bar X - \mu \rvert > \sqrt{\frac{24 \mu \log(2 / \delta)}{n}} + \frac{48 \log(2 / \delta)}{n}) \le 4 \delta$,
    \item $\mathbb P(\bar X \ge \frac{3\mu}{2} + \frac{48 \log(2 / \delta)}{n}) \le 3 \delta$.
\end{enumerate}
\end{lem}

The following lemma is a direct corollary of Lemma \ref{lem:hpver1}. It can be proved by following the proof of Lemma 4.1 in \citep[Appendix A]{agrawal2019mnl}. Here ``$\land$'' means logical and. 

\begin{lem} \label{lem:hpver2} Consider $n$ i.i.d. geometric random variables $X_1, \ldots, X_n$ with expectation $\mathbb E[X_i] = \mu \le 1$. Let $\bar X = \frac{1}{n}\sum_{i = 1}^n X_i$ and $\epsilon = \sqrt{\frac{48 \bar X \log(2 / \delta)}{n}} + \frac{48 \log(2 / \delta)}{n}$. 
Then we have
$$\mathbb P(\{\bar X - \epsilon \le \mu \le \bar X + \epsilon\} \land \{\epsilon \le \sqrt{\frac{196 \mu \log(2 / \delta)}{n}} + \frac{196 \log(2 / \delta)}{n}\}) \ge 1 - 13\delta.$$
\end{lem}

We state another concentration inequality to the geometric random variables. The following inequality is focused on the upper tail of the geometric random variables. 

\begin{lem}[\citet{janson2018tail}, Theorem 2.1] \label{lem:geoshifted} Consider $n$ independent ``shifted'' geometric random variables $X_1, \ldots, X_n$ that $\mathbb P(X_i = k) = p_i(1 - p_i)^{k-1}$. Let $p_* = \min_{1 \le i \le n} p_i > 0, X = \sum_{i = 1}^n X_i, \mu = \mathbb E[X]$. We have 
$\mathbb P(X \ge \lambda \mu) \le e^{-p_* \mu(\lambda - 1 - \ln \lambda)}.$
\end{lem}

\section{Proofs for Section \ref{sec:reduction}}

\subsection{Proof of Proposition \ref{prop:topk}}

\label{app:prooftopk}

\begin{proof}[Proof of Proposition \ref{prop:topk}] In \citep[Section 2.1]{rusmevichientong2010dynamic}, it was shown that the optimal revenue is 
\begin{align*}
    \theta^* = \max\{\theta \in \mathbb R : \max_{S\subseteq [N]: \lvert S \rvert \le K} \sum_{i \in S} v_i(r_i - \theta) \ge \theta\}.
\end{align*}
Let $S = \argmax_{S \subseteq [N]: \lvert S \rvert \le K}\{\sum_{i \in S} v_i(r_i - \theta^*)\}$. By the above equation, we have 
$$\sum_{i \in S} u_i = \sum_{i \in S} v_i(r_i - \theta^*) \ge \theta^*.$$
Next we show the above ``$\ge$'' is actually ``$=$''. Suppose instead, it is ``$>$'', then we have 
\begin{align*}
    \sum_{i \in S} v_i (r_i - \theta^*) &> \theta^*, \\
    \sum_{i \in S} v_i r_i &> (1 + \sum_{i \in S} v_i) \theta^*, \\
    \frac{\sum_{i \in S} v_i r_i}{1 + \sum_{i \in S} v_i} &> \theta^*,
\end{align*}
which implies that $R(S, \bm v) > \theta^*$ and contradicts to that $\theta^*$ is the optimal revenue. As a result, we have $\theta^* = \sum_{i \in S} u_i$. Note that when ``$=$'' holds, by repeating the above argument, we have $R(S, \bm v) = \theta^*$ and thus $S = S^*$ by Assumption \ref{assum:uniqueopt}. Therefore, 
\begin{align*}
    S^* = \argmax_{S \subseteq [N] : \lvert S \rvert \le K} \{\sum_{i \in S} u_i\}. 
\end{align*}
It is clear that 
\begin{align*}
    \argmax_{S \subseteq [N] : \lvert S \rvert \le K} \{\sum_{i \in S} u_i\} = \mathcal F([N], K, \bm u),
\end{align*}
because to maximize the sum of scores under the capacity constraint, it suffices to pick all items with positive and top-$K$ scores. 
\end{proof}

\subsection{Proof of Lemma \ref{lem:peframework}}

\label{app:proofpeframework}

We prove Lemma \ref{lem:peframework} to show the sample complexity guarantee of Algorithm \ref{algo:peframework}. We first reveal the relation between the gap of advantage score and the suboptimality gap of each item. 

\begin{lem}[Relation between $\gap_i$ and $u_i$] \label{lem:diui3} For items $i, j \in [N]$, we have the following statements. \begin{enumerate}
    \item If $i \in S^*, j \notin S^*$, then $\gap_i \le u_i - u_j$. In addition, $\gap_i \le u_i$.
    \item If $i \notin S^*, j \in S^*$, then $\gap_i \le u_j - u_i$. If in addition $\lvert S^* \rvert < K$, then $\gap_i \le -u_i$. 
  \end{enumerate}
\end{lem}

\begin{proof} For (a), let $S = (S^* \setminus \{i\}) \cup \{j\}$. Note that $\gap_i \le R(S^*, \bm v) - R(S, \bm v)$, so by Lemma \ref{lem:comparison}, we have
$$\gap_i \le (1 + \sum_{l \in S } v_l)(R(S^*, \bm v) - R(S, \bm v)) = u_i - u_j.$$ 
Let $S = S^* \setminus \{i\}$ and repeat the previous argument, we have $\gap_i \le u_i$. For (b), let $S = (S^* \setminus \{j\}) \cup \{i\}$. Similarly, we have $$\gap_i \le (1 + \sum_{l \in S} v_l)(R(S^*, \bm v) - R(S, \bm v)) = u_j - u_i.$$
When $\lvert S^* \rvert < K$, we let $S = S^* \setminus \{j\}$ and repeat the previous argument to obtain $\gap_i \le -u_i$. 
\end{proof}

\begin{lem}[Revenue Comparison Lemma] \label{lem:comparison}
Let $S \subseteq [N]$ be an assortment. Then we have $(1 + \sum_{i \in S} v_i) (\theta^* - R(S, \bm v)) = \sum_{i \in S^* \setminus S} u_i - \sum_{i \in S \setminus S^*} u_i$.
\end{lem}

\begin{proof} We have 
\begin{align*}
(1 + \sum_{i \in S} v_i)(\theta^* - R(S, \bm v))  &= (1 + \sum_{i \in S} v_i)\theta^* - \sum_{i \in S} v_i r_i \\
&= \theta^* - \sum_{i \in S} v_i(r_i - \theta^*) \\
(\text{Proposition \ref{prop:topk}})&= \sum_{i \in S^*} u_i - \sum_{i \in S} u_i \\
&= \sum_{i \in S^* \setminus S} u_i - \sum_{i \in S \setminus S^*} u_i. \qedhere
\end{align*}
\end{proof}

Next we prove Lemma \ref{lem:peframework} in twofold. First, we analyze the guarantees of the accept-reject stage at Lines \ref{loc:arstart}-\ref{loc:arend} in Algorithm \ref{algo:peframework}.

\begin{lem}[Accept-Reject] \label{lem:ar} In phase $k$, before Line \ref{loc:arstart}, we have $A^{(k-1)} \subseteq S^* \subseteq A^{(k-1)} \sqcup B^{(k-1)}$ and $u_i \in [\check \xi_i, \hat \xi_i], \hat \xi_i - \check \xi_i \le \frac{\epsilon_k}{2}$ for $i \in B^{(k-1)}$. Then after Line \ref{loc:arend}, we have $A^{(k)} \subseteq S^* \subseteq A^{(k)} \sqcup B^{(k)}$ and $B^{(k)} \subseteq \{i \in [N] : \Delta_i \le \epsilon_k\}$.  
\end{lem}

To facilitate readability, we divide the lemma into two lemmas and prove them separately. 

\begin{proof}[Proof of Lemma \ref{lem:ar}] We combine Lemmas \ref{lem:ar1} and \ref{lem:ar2}.
\end{proof}

\begin{lem} \label{lem:ar1} Under the context of Lemma \ref{lem:ar}, we have $A^{(k)} \subseteq S^* \subseteq A^{(k)} \sqcup B^{(k)}$. 
\end{lem}

\begin{lem} \label{lem:ar2} Under the context of Lemma \ref{lem:ar}, we have $B^{(k)} \subseteq \{i \in [N] : \Delta_i \le \epsilon_k\}$. 
\end{lem}

We prove these two lemmas. Let $B_{\mathrm{acc}}^1 = \{b \in B^{(k-1)} : \check \xi_b > 0 \}, B_{\mathrm{rej}}^1 = \{b \in B^{(k-1)} : \hat \xi_b < 0\}$. If $\lvert B^{(k-1)}  \rvert > M$,  we let $B_{\mathrm{acc}}^2 = \{b \in B^{(k-1)} : \check \xi_b > \beta \}, B_{\mathrm{rej}}^2 =  \{b \in B^{(k-1)} : \hat \xi_b < \alpha \}$, where $M, \alpha, \beta$ are defined in Algorithm \ref{algo:peframework}.

\begin{proof}[Proof of Lemma \ref{lem:ar1}] We recall that the notion ``$\sqcup$'' requires $A^{(k)} \cap B^{(k)} = \emptyset$, so we show this first. This follows directly from $A^{(k)} = A^{(k-1)} \cup B_{\mathrm{acc}}, B^{(k)} \subseteq B^{(k-1)} \setminus B_{\mathrm{acc}},$ and $A^{(k-1)} \cap B^{(k-1)} = \emptyset$. 
 
 Next, we show $A^{(k)} \subseteq S^* \subseteq A^{(k)} \sqcup B^{(k)}$. It suffices to show $i \in S^* \setminus A^{(k-1)}$ for $i \in B_{\mathrm{acc}}$ and $i \notin S^* \setminus A^{(k-1)}$ for $i \in B_{\mathrm{rej}}$. Suppose $\lvert B^{(k-1)} \rvert \le M$. Since $u_b \ge \check \xi_b$, we have
 $$B^1_{\mathrm{acc}} \subseteq \{b \in B^{(k-1)} :  u_b > 0 \} \subseteq \mathcal F(B^{(k-1)}, M, \bm u) = S^* \setminus A^{(k-1)},$$
 which implies $A^{(k)} \subseteq S^*$. We have $u_b \le \hat \xi_b < 0$ for $b\in B_{\mathrm{rej}}$, which implies $b \notin S^*$ and thus $S^* \subseteq (A^{(k-1)} \sqcup B^{(k - 1)}) \setminus B_{\mathrm{rej}} = A^{(k)} \sqcup B^{(k)}$. 
 
 Now consider $\lvert B^{(k-1)} \rvert > M$. For every $i \in B_{\mathrm{acc}}$, since $i \in B^1_{\mathrm{acc}}$, we have $u_i > 0$. Since $i \in B^2_{\mathrm{acc}}$, we have $u_i \ge \check \xi_i > \beta$. By the definition of $\beta$, we know that $\#\{b \in B^{(k-1)} : u_b \ge u_i\} \le M$. Therefore, $u_i$ is positive and top-$M$. Thus $i \in \mathcal  F(B^{(k-1)}, M, \bm u) = S^* \setminus A^{(k-1)}$. 
 
 For every $i \in B_{\mathrm{rej}}$, if $i \in B^1_{\mathrm{rej}}$, then $u_i \le \hat \xi_i < 0$ is negative, thus $i \notin S^*$. Otherwise we have $i \in B^2_{\mathrm{rej}}$. By the definition of $\alpha$, we have that $u_i \le \hat \xi_i < \alpha$ and thus $\#\{b \in B^{(k-1)} : u_b > u_i\} > M$. Therefore, $u_i$ is not top-$M$. Thus $i \notin S^* \setminus A^{(k-1)}$. 
\end{proof}

\begin{proof}[Proof of Lemma \ref{lem:ar2}] We show $B^{(k)} \subseteq \{i \in B^{(k-1)} : \Delta_i \le \epsilon_k\}$ by showing that $\Delta_i > \epsilon_k$ implies $i \notin B^{(k)}$. Fix $i \in B^{(k-1)}$ such that $\Delta_i > \epsilon_k$. 

1. Suppose $i \in S^*$. We will show that $i \in B_{\mathrm{acc}}$. By Lemma \ref{lem:diui3}, we have $\gap_i \le u_i$ and thus
  $$\check{\xi}_i \ge \hat{\xi}_i - \frac{\epsilon_k}{2} \ge u_i - \frac{\epsilon_k}{2} \ge \gap_i - \frac{\epsilon_k}{2} \ge \epsilon_k - \frac{\epsilon_k}{2} = \frac{\epsilon_k}{2} > 0,$$
  which implies $i \in B^1_{\mathrm{acc}}$. Note that  when $\lvert B^{(k-1)} \rvert \le M$, we have $B_{\mathrm{acc}} = B_{\mathrm{acc}}^1$ and thus we conclude.

  When $\lvert B^{(k-1)}\rvert > M$, it remains to show $i \in B^2_{\mathrm{acc}}$. By the definition of $\beta$, it suffices to show $\#\{j \in B^{(k-1)} : \check \xi_i > \hat \xi_j\} \ge \lvert B^{(k-1)} \rvert - M$, which is equivalent to  $\#\{j \in B^{(k-1)} : \check \xi_i \le \hat \xi_j\} \le M$. 
  
  For every $j \in B^{(k - 1)}$, if $\hat{\xi}_j \ge \check{\xi}_i$, then we have $u_j + \frac{\epsilon_k}{2} \ge \hat{\xi}_j \ge \check{\xi}_i \ge u_i - \frac{\epsilon_k}{2}$. Therefore, $u_j \ge u_i - \epsilon_k \ge u_i - \gap_i$. In summary, we have $\{j \in B^{(k - 1)} : \hat{\xi}_j > \check{\xi}_i\} \subseteq \{j \in B^{(k-1)} : \xi_j \ge \xi_i - \gap_i\}.$ By Lemma \ref{lem:diui3}, we have $\xi_i - \xi_{j'} \ge \gap_i$ for every $j' \in [N] \setminus S^*$. Thus $\{j \in B^{(k - 1)} : \xi_j \ge \xi_i - \gap_i\} \subseteq S^*$. Recall that $A^{(k - 1)} \subseteq S^* \subseteq A^{(k-1)} \sqcup B^{(k-1)}$.  So $\{j \in B^{(k - 1)} : \hat{\xi}_j \ge \check{\xi}_i\} \subseteq S^* \setminus A^{(k - 1)}$ and thus $\#\{j \in B^{(k)} : \hat{\xi}_j \ge \check{\xi}_i\} \le K - \lvert A^{(k - 1)} \rvert = M$, which completes the proof.

2.  Suppose $i \notin S^*$. We will show that $i \in B_{\mathrm{rej}}$. Suppose $\lvert S^* \rvert < K$. By Lemma \ref{lem:diui3}, we have $\Delta_i \le - \xi_i$. Therefore, $$\hat \xi_i \le \xi_i + \frac{\epsilon_k}{2} < -\epsilon_k + \frac{\epsilon_k}{2} < 0,$$
  which implies $i \in B_{\mathrm{rej}}^1$. 
  
  Now consider $\lvert S^* \rvert = K$. Since $i \notin S^*$ and $i \in B^{(k-1)}$, we must have $\lvert B^{(k-1)} \rvert \ge M + 1 > M$. In the following, we show that $i \in B_{\mathrm{rej}}^2$. By the definition of $\alpha$, it suffices to show $\#\{j \in B^{(k-1)} : \check \xi_j > \hat \xi_i\} \ge M$. For every $j \in S^* \cap B^{(k-1)}$, by Lemma \ref{lem:diui3}, we have $\Delta_i \le u_j - u_i$. Therefore, we have $\check \xi_j \ge u_j - \frac{\epsilon_k}{2} \ge \Delta_i + u_i - \frac{\epsilon_k}{2} \ge \Delta_i - \frac{\epsilon_k}{2} + \hat \xi_i - \frac{\epsilon_k}{2} > \hat \xi_i$, which implies $\{j \in B^{(k-1)} : \check \xi_j > \hat \xi_i\} \supseteq S^* \cap B^{(k-1)}$ and thus $\#\{j \in B^{(k-1)} : \check \xi_j > \hat \xi_i\} \ge M$. 
\end{proof}

\begin{proof}[Proof of Lemma \ref{lem:peframework}] By a union bound, the probability that \textsf{EST} returns confidence intervals $u_i \in [\check \xi_i, \hat \xi_i], \hat \xi_i - \check \xi_i \le \frac{\epsilon_k}{2}$ within $C_{\textsf{EST}} \cdot \frac{ \lvert B^{(k-1)} \rvert \log(N / \delta^{(k)})}{\epsilon_k^2}$ time steps for every phase $k \in \mathbb N$ is at least 
\begin{align*}
    \prod_{k = 1}^\infty (1 - \delta^{(k)}) = \prod_{k = 1}^\infty (1 -  \frac{\delta}{3 k^2}) \ge 1 - \sum_{k = 1}^\infty \frac{\delta}{3 k^2} \ge 1 - \frac{\delta}{3} \frac{\pi^2}{6} \ge 1 - \delta.
\end{align*}
We condition on the above event. Note that $A^{(0)} \subseteq S^* \subseteq A^{(0)} \sqcup B^{(0)}$. By combining Lemma \ref{lem:ar} with an induction over phases, we can show that $A^{(k)} \subseteq S^* \subseteq A^{(k)} \sqcup B^{(k)}$ and $B^{(k)} \subseteq \{i \in [N]: \Delta_i \le \epsilon_k\}$ for every phase $k$. Therefore, when $M = 0$, the algorithm returns the optimal assortment $S^*$. The sample complexity of \textsf{SAR-MNL} with \textsf{EST} is
\begin{align*}
   T &\lesssim \sum_{k = 1}^\infty \lvert B^{(k-1)} \rvert \cdot \frac{C_{\textsf{EST}} \log(N /  \delta^{(k)})}{\epsilon_k^2}  \\
    &\lesssim \sum_{k = 1}^\infty \left( \sum_{i \in [N]} \mathbb I\{\Delta_i \ge \epsilon_k\} \right) \cdot \frac{C_{\textsf{EST}} \log(N k / \delta)}{\epsilon_k^2} \\
    &= \sum_{i \in [N]} \sum_{k = 1}^\infty \frac{C_{\textsf{EST}} \log(N k / \delta)}{\epsilon_k^2}  \cdot \mathbb I\{\Delta_i \ge \epsilon_k\} \\
    &\lesssim\sum_{i \in [N]} \sum_{k = 1}^{\lceil \log \Delta_i^{-1} \rceil} \frac{C_{\textsf{EST}} \log(N k / \delta)}{\epsilon_k}  \cdot  \\
    &\lesssim C_{\textsf{EST}} \cdot \sum_{i \in [N]} \frac{\log N + \log \log \Delta_i^{-1} + \log \delta^{-1}}{\Delta_i^2}. \qedhere
\end{align*}
\end{proof}

\subsection{Proof of Lemma \ref{lem:naiveguarantee}}

\label{app:naiveguarantee}

Before proving the lemma, we first specify the skipped formulas in Algorithm \ref{algo:naiveest}. Let $C_0 = 196, C_2 = 1024, \delta = \frac{\delta_0}{15 N}$. We define $\tau = \frac{C_2 C_0 \logterm}{\epsilon^2}$ and
\begin{align*}
    \check{v}_i = 0 \lor (\bar{v}_i - \sigma(v_i)), \hat{v}_i = 1 \land (\bar{v}_i + \sigma(v_i)), &\quad \sigma(v_i) = \sqrt{\frac{48 \bar{v}_i \logterm}{T_i}} +  \frac{48 \logterm}{T_i},
\end{align*}
where $T_i = K \tau$ is the number of offering. For each item $i \in A \cup B$, we define $\bar v_i = \frac{n_i}{T_i}$, where $n_i$ is the total number of time steps with outcome ``item $i$''. One may realize that ``keep offering until no purchase'' is the same as the epoch-based offering in \citep{agrawal2019mnl} and that \textsf{EST-NAIVE} uses a simplified version by only offering singletons. We adopt the notions, calling it ``epoch'' and referring $T_i$ as the number of epochs. 

Next we give a proof of the sample complexity guarantee using previous results in \citep{agrawal2019mnl}. Our proof frequently uses the big-$O$ notations to suppress the constants, whose exact values can be calculated by following the proofs in Appendix \ref{app:perefined}. 

\begin{proof}[Proof of Lemma \ref{lem:naiveguarantee}] 1. We prove that \textsf{EST-NAIVE} returns the confidence intervals $u_i \in [\check \xi_i, \hat \xi_i]$ with high probability. By Lemma 4.1 in \citep{agrawal2019mnl}, we have that $v_i \in [\check v_i, \hat v_i]$ and $\hat v_i - \check v_i \le \widetilde{O}(\sqrt{v_i / T_i}) = \widetilde O(\frac{\epsilon v_i}{K})$ with probability $1 - O(N \delta)$. By Lemma 4.2 in \citep{agrawal2019mnl}, we find that $\check \theta \le \theta^* \le \hat \theta$ if $\check v_i \le v_i \le \hat v_i$ for every $i \in A \cup B$. Furthermore, for the assortment $S = \argmax_{S \subseteq A \cup B : \lvert S \rvert \le K} R(S, \hat v)$, we have 
\begin{align} \notag
    R(S, \hat v) - R(S, \check v) &= \frac{\sum_{i \in S} \hat v_i r_i}{1 + \sum_{i \in S} \hat v_i} - \frac{\sum_{i \in S} \check v_i r_i}{1 + \sum_{i \in S} \check v_i} \\ \notag
    &\le \frac{\sum_{i \in S} \hat v_i r_i}{1 + \sum_{i \in S} \check v_i} - \frac{\sum_{i \in S} \check v_i r_i}{1 + \sum_{i \in S} \check v_i} \\ \label{eq:naive1}
    &\le \sum_{i \in S} (\hat v_i - \check  v_i) r_i \\ \notag
    &\le \sum_{i \in S} (\hat v_i - \check  v_i) \\ \notag
    &\le K O(\frac{\epsilon}{K}) \\ \notag
    &\le O(\epsilon). \notag
\end{align}
Note that $R(S, \hat v) = \hat \theta$ and $R(S, \check v) \le \check \theta$, so we conclude that $\hat \theta - \check \theta \le O(\epsilon)$. Finally, we note that $u_i = v_i(r_i - \theta^*)$ and that $v_i \in [0, 1], (r_i - \theta^*) \in [-1, 1]$. Therefore, we have 
$\hat \xi_i - \check \xi_i \le \lvert \hat v_i - \check v_i \rvert + \lvert \hat \theta - \check \theta \rvert \le O(\epsilon)$.

2. We conclude by showing \textsf{EST-NAIVE} achieves $C_{\textsf{EST}} = O(K^2)$ in Lemma \ref{lem:peframework}. When we keep offering a singleton assortment $\{i\}$ until the outcome ``no purchase'' occurs, it will take us $1 + v_i \le 2$ time steps in expectation. So in expectation, \textsf{EST-NAIVE} uses $$K \tau \sum_{i \in A \cup B} (1 + v_i) \le 2 K \lvert A \cup B \rvert \tau \le 2 K^2 \tau$$ time steps. Using the concentration inequalities, we can turn the expectation argument into a high probability one, showing that \textsf{EST-NAIVE} returns in $O(K^2 \tau)$ time steps with probability at least $1 - O(\delta)$. Thus we prove that $C_{\textsf{EST}} = O(K^2)$ for \textsf{EST-NAIVE}. 
\end{proof}

Finally, we discuss two questions: why the procedure only offers singletons and why the accuracy needs to be $\frac{\epsilon}{K}$. For the first question, we discuss its optimality under the epoch-based offering framework \citep{agrawal2019mnl}, which is used by almost all previous MNL-bandit work. Under this framework, the accuracy of our estimation to $v_i$ solely depends on $T_i$, the number of epochs that offers item $i$. 

Let us consider that all items have $v_i = \Theta(1)$ and compare two offering schemes for an assortment $S$: (i) offer $S$ for an epoch; (ii) for each item $i \in S$, offer the singleton assortment $\{i\}$ for an epoch.  Both offering schemes increase $T_i$ by $1$ for every $i \in S$ and thus lead to the same accuracy. Moreover, in expectation, the number of time steps used by the first scheme is $(1 + \sum_{i \in S} v_i)$ and that used by the second scheme is $\sum_{i \in S} (1 + v_i)$. When $v_i = \Theta(1)$, we have $(1 + \sum_{i \in S} v_i) \asymp \sum_{i \in S} (1 + v_i)$. As a result, both schemes use a similar number of time steps, so we do not benefit from offering an assortment with size greater than $1$, i.e. offering singletons could be enough. 

For the second question, we consider that all items have $v_i = \Theta(\frac{1}{K})$. We note that if we need to estimate $u_i$ to a given accuracy $\epsilon$, we need to estimate $\theta^*$ to such accuracy: $\hat \theta - \check \theta \le \epsilon$. We observe that when $v_i = \Theta(\frac{1}{K})$, the step in Eq. (\ref{eq:naive1}) is almost  \emph{tight}, because $$(1 + \sum_{i \in S} \check v_i) \asymp (1 + \sum_{i \in S} v_i) \asymp (1 + \sum_{i \in S} \frac{1}{K}) \asymp (1 + \lvert S \rvert \frac{1}{K}) \asymp 1.$$
To estimate $\theta^*$ to the accuracy $\epsilon$, by Eq. (\ref{eq:naive1}), we need that 
\begin{align}
    \hat \theta - \check \theta \le \cdots \le \sum_{i \in S} (\hat v_i - \check v_i) \le \cdots \le \epsilon.  \label{eq:naive2}
\end{align}
Since $\lvert S \rvert$ can be $O(K)$, we need to estimate each $v_i$ to the accuracy $\frac{\epsilon}{K}$ in order to achieve Eq. (\ref{eq:naive2}), which suggests that estimating to the accuracy $\frac{\epsilon}{K}$ could be necessary.

Note that we explain these two questions under different instances, namely $v_i = \Theta(1)$ and $v_i = \Theta(\frac{1}{K})$, so it is still possibly to design an estimation procedure that adapts to these different instances. Actually, this is what we show in Section \ref{sec:adaptive} and Appendix \ref{app:optimal}. 

\section{Proofs for Section \ref{sec:reduced}}

The following lemma shows that the maximization of the reduced revenue function is monotonic in its parameters and thus we can use Eq. (\ref{eq:esttheta}) to compute the confidence interval of the optimal revenue. 

\begin{lem}[Monotonicity] \label{lem:mono} Assume $A \subseteq S^* \subseteq A \sqcup B$ and let $M = \min\{K - \lvert A \rvert, \lvert B \rvert \}$. Suppose $\srew \in [\check \srew, \hat \srew]$ and $\nu_i \in [\check \nu_i, \hat \nu_i]$ for every $i \in B$. Let $\check \theta, \hat \theta, \check \xi_i, \hat \xi_i$ be those defined in Eqs. (\ref{eq:esttheta}) (\ref{eq:estu}). Then we have $\theta^* \in [\check \theta, \hat \theta]$ and $\xi_i \in [\check \xi_i, \hat \xi_i]$ for $i \in B$.
\end{lem}

\begin{proof} First, we show $\theta^* \in [\check \theta, \hat \theta]$. We will only show $\theta^* \le \hat \theta$, since the proof of $\theta^* \ge \check \theta$ is similar. By Eq. (\ref{eq:esttheta}), we have 
\begin{align*}
    \hat{\theta} = \max_{S \subseteq B: \lvert S \rvert \le M} R(S, \hat{\nu}, \hat{\srew})  \ge R(S^* \setminus A, \hat \nu, \hat \srew). 
\end{align*}
Also we have 
\begin{align*}
    \theta^* = R(S^* \setminus A, \nu, \srew) &= \frac{\srew + \sum_{i \in S^* \setminus A} \nu_i r_i}{1 + \sum_{i \in S^* \setminus A} \nu_i}, \\
    (1 + \sum_{i \in S^* \setminus A} \nu_i) \theta^* &= \srew + \sum_{i \in S^* \setminus A} \nu_i r_i, \\
    \theta^* &= \srew + \sum_{i \in S^* \setminus A} \nu_i (r_i - \theta^*).
\end{align*}
By Proposition \ref{prop:topk}, we have $S^* = \mathcal F([N], K, \bm u)$, so $u_i \ge 0$ for $i \in S^*$, thus $r_i \ge \theta^*$ for $i \in S^*$. Therefore, 
\begin{align*}
    \hat \srew + \sum_{i \in S^* \setminus A} \hat \nu_i(r_i - \theta^*) &\ge \srew + \sum_{i \in S^* \setminus A} \nu_i (r_i - \theta^*) = \theta^*, \\
    \hat \srew + \sum_{i \in S^* \setminus A} \hat \nu_i r_i &\ge (1 + \sum_{i \in S^* \setminus A} \hat \nu_i) \theta^*, \\
    \frac{\hat \srew + \sum_{i \in S^* \setminus A} \hat \nu_i r_i}{1 + \sum_{i \in S^* \setminus A} \hat \nu_i} &\ge \theta^*, \\
    R(S^* \setminus A, \hat \nu, \hat \srew) &\ge \theta^*.
\end{align*}
And we conclude that $\hat \theta \ge \theta^*$. Second, we show $\xi_i \in [\check \xi_i, \hat \xi_i]$. Recall that $\xi_i = \nu_i(r_i - \theta^*)$. We conclude by noting that $(r_i - \theta^*) \in [-1, 1]$ and that $\nu_i \in [0, 1]$.
\end{proof}

\subsection{Proof of Proposition \ref{prop:ind}}

\label{app:propind}

\begin{proof}[Proof of Proposition \ref{prop:ind}] Statement (a) can be proved by noting that $\mathbb P(z = r_i) = \frac{v_i}{1 + \sum_{j \in \sset} v_j}$. Statement (c) can be proved by noting that $(E_\ell - 1)$ follows a geometric distribution with parameter $p = \frac{1 + \sum_{i \in \sset} v_i}{1 + \sum_{i \in \sset} v_i + \sum_{i \in S} v_i}$, so it has mean $\mathbb E[E_\ell - 1] = \frac{1 - p}{p} = \frac{\sum_{i \in S} v_i}{1 + \sum_{i \in \sset} v_i} = \sum_{i \in S} \nu_i$.

Now we prove statement (b). When $\sset = \emptyset$, it was the same as Corollary A.1 in \citep{agrawal2019mnl}. We note that $\sset = \emptyset$ case implies $\sset \ne \emptyset$ case, because the distribution of $x_i$ when we offer the assortment $\sset \sqcup S$ under parameter $\bm v$ and stop at outcomes $\sset \sqcup \{0\}$ is the same as when we offer $S$ under parameter $\nu$ and stops at outcome $0$. 
\end{proof}

The next lemma bounds the sample complexity when using the generalized epoch-based offering procedure using statement (c) in last proposition.

\begin{lem}[Sum of Epoch Lengths] \label{lem:sumofepochlength} Suppose we independently explore $L \ge \log(1 / \delta)$ epochs using Algorithm \ref{algo:explore} and the expected length of each epoch $\ell \in [L]$ is $\mathbb E[E_\ell] \le 3$. Let $T = \sum_{\ell = 1}^L E_\ell$ be the total number of used time steps. With probability at least $1 - \delta$, we have $T \le 8 \mathbb E[T] \le 24L$. 
\end{lem}

\begin{proof} Note that $\{E_\ell - 1\}_{\ell = 1}^L$ are independent geometric random variables with mean $\mathbb E[E_\ell] \le 3$. Let $\lambda = 8$. Then $\lambda - 1 - \ln \lambda \ge 3$. Since $\mathbb E[T] \ge L$, by Lemma \ref{lem:geoshifted}, we have 
\[
    \mathbb P(T \ge 8 \mu) \le e^{-p_* \mu(\lambda - 1 - \ln \lambda)} 
    \le  e^{-\frac{1}{3}L(\lambda - 1 - \ln \lambda)} 
    \le e^{-L} 
    \le \delta. \qedhere 
\]
\end{proof}

\subsection{Enhanced Version of Lemma \ref{lem:peframework}}

\label{app:enhancedpeframework}

We show that if we assume \textsf{EST} returns an estimation of the reduced advantage score $\xi_i$, we can still obtain a similar sample complexity guarantee as that in Lemma \ref{lem:peframework}.

\begin{lem}[Lemma \ref{lem:peframework} enhanced] \label{lem:enhancedpeframework} Assume $A^{(k-1)} \subseteq S^*$. Suppose with probability at least $1 - \delta^{(k)}$, \textsf{EST} (a) returns in $C_{\textsf{EST}} \cdot \frac{ \lvert B^{(k-1)} \rvert \log(N / \delta^{(k)})}{\epsilon_k^2}$ time steps in phase $k$, and (b) $\xi_i \in [\check \xi_i, \hat \xi_i]$ and $\hat \xi_i - \check \xi_i \le \frac{\epsilon_k}{2}$ for every $i \in B^{(k-1)}$, where $\xi_i = \frac{u_i}{1 + \sum_{j \in A^{(k-1)}} v_j}$ is the reduced score. Then \textsf{SAR-MNL} with \textsf{EST} is $\delta$-PAC with sample complexity $C_{\textsf{EST}} \cdot O(\sum_{i \in [N]} \frac{\log N + \log \delta^{-1} + \log \log \gap_i^{-1}}{\Delta_i^2})$.
\end{lem}

\begin{proof} We replace Lemma \ref{lem:ar} with Lemma \ref{lem:arenhanced} in the proof of Lemma \ref{lem:peframework}.
\end{proof}

\begin{lem} \label{lem:arenhanced} In phase $k$ in Algorithm \ref{algo:peframework}, suppose we have $A^{(k-1)} \subseteq S^* \subseteq A^{(k-1)} \sqcup B^{(k-1)}$, and after invoking \textsf{EST} in phase $k$, we have $\xi_i \in [\check \xi_i, \hat \xi_i], \hat \xi_i - \check \xi_i \le \frac{\epsilon_k}{2}$ for $i \in B^{(k-1)}$. Then after Line \ref{loc:arend} , we have $A^{(k)} \subseteq S^* \subseteq A^{(k)} \sqcup B^{(k)}$ and $B^{(k)} \subseteq \{i \in [N] : \Delta_i \le \epsilon_k\}$.  
\end{lem}

\begin{proof} Let $\sset = A^{(k-1)}$ in Lemma \ref{lem:diui2}. We replace Lemma \ref{lem:diui3} with Lemma \ref{lem:diui2}  and replace the score $u_i$ with the score $\xi_i = \frac{u_i}{1 + \sum_{j \in A^{(k-1)}} v_j}$  in the proof of Lemma \ref{lem:ar} to prove the lemma. 
\end{proof}

\begin{lem}[Relation between $\gap_i$ and $\xi_i$] \label{lem:diui2} For a set $\sset \subseteq S^*$ and an item $i \in [N] \setminus \sset$, we define the reduced advantage score $\xi_i = \frac{u_i}{1 + \sum_{l \in \sset} v_i}$. Then for items $i, j \in [N] \setminus \sset$, we have 
\begin{enumerate}
    \item If $i \in S^*, j \notin S^*$, then $\gap_i \le \xi_i - \xi_j$. In addition, $\gap_i \le \xi_i$. 
    \item If $i \notin S^*, j\in S^*$, then $\gap_i \le \xi_j - \xi_i$.  If in addition $\lvert S^* \rvert < K$, then $\gap_i \le -\xi_i$. 
\end{enumerate}
\end{lem}

\begin{proof} For (a), let $S = (S^* \setminus \{i\}) \cup \{j\}$. Note that $\gap_i \le R(S^*, \bm v) - R(S, \bm v)$, so by Lemma \ref{lem:comparison}, we have
$$(1 + \sum_{l \in \sset} v_l) \gap_i \le (1 + \sum_{l \in S } v_l)(R(S^*, \bm v) - R(S, \bm v)) = u_i - u_j.$$ 
Note that $\xi_i = \frac{u_i}{1 + \sum_{t \in \sset} v_t}$ for $i \notin \sset$, so $\gap_i \le \xi_i - \xi_j$. Let $S = S^* \setminus \{i\}$ and repeat the previous argument, we have $\gap_i \le \xi_i$. For (b), let $S = (S^* \setminus \{j\}) \cup \{i\}$. Similarly, we have $$(1 + \sum_{l \in \sset} v_t) \gap_i \le (1 + \sum_{l \in S} v_t)(R(S^*, \bm v) - R(S, \bm v)) = u_j - u_i.$$
Thus $\gap_i \le \xi_j -\xi_i$. When $\lvert S^* \rvert < K$, we let $S = S^* \setminus \{j\}$ and repeat the previous argument to obtain $\gap_i \le - \xi_i$. 
\end{proof}

\subsection{Estimation Procedure with Generalized Epoch-based Offering}

\label{app:reducedguarantee}

We present an estimation procedure \textsf{EST-REDUCED} (Algorithm \ref{algo:reduced})  to demonstrate the power of the generalized epoch-based offering.  

\begin{lem} \label{lem:reducedguarantee} There is a $\delta$-PAC algorithm with sample complexity $\widetilde{O}(\sum_{i = 1}^N \frac{K}{\gap_i^2})$ using only techniques in Sections \ref{sec:reduction} and \ref{sec:reduced}.
\end{lem}

\begin{proof} We claim that \textsf{SAR-MNL} with \textsf{EST-REDUCED} can serve as the algorithm in the lemma. The statement (b) in Lemma \ref{lem:reduced} shows that \textsf{EST-REDUCED} can serve as the estimation procedure \textsf{EST} in Lemma \ref{lem:enhancedpeframework} and (a) further shows that \textsf{EST-REDUCED} satisfies $C_{\textsf{EST}} =  O(K)$. Thus we conclude by Lemma \ref{lem:enhancedpeframework}.
\end{proof}

\begin{algorithm}[t]
\label{algo:reduced}
\caption{$\textsf{EST-REDUCED}(A, B, \delta_0, \epsilon)$: Estimation of $\xi_i$ for $i \in B$}
$C_0 = 196, C_2 = 1024, \delta = \frac{\delta_0}{15 N}, \tau = \frac{C_2 C_0 \logterm}{\epsilon^2}$, $Z \gets A, n_\sset = T_\sset = 0, \forall i \in B: n_i = T_i = 0$\; \label{loc:reducedtau}
    $\forall i \in B: \textsf{Explore}(\{i\})$ for $K \tau$ epochs\;
    Compute $\check{\srew}, \hat{\srew}, \check{\nu}_i, \hat{\nu}_i, \check \theta, \hat \theta, \check{\xi}_i, \hat{\xi}_i$ by Eqs. (\ref{eq:estz}) (\ref{eq:estv}) (\ref{eq:esttheta}) (\ref{eq:estu}) for $i \in B$, \Return{$\{\check \xi_i, \hat \xi_i\}_{i \in B}$}\; \label{loc:est222}
\end{algorithm}

\begin{lem}[\textsf{EST-REDUCED}] \label{lem:reduced} Assume $A \subseteq S^* \subseteq A \sqcup B$. With probability $1 - \delta_0$, (a) \textsf{EST-REDUCED} returns in $O(K \lvert B \rvert \tau)$ time steps, where $\tau = O(\frac{\log N/\delta_0}{\epsilon^2})$ as defined in Algorithm \ref{algo:reduced}; (b) $\xi_i = \frac{u_i}{1 + \sum_{j \in A} v_j} \in [\check \xi_i, \hat \xi_i]$ and $\hat \xi_i - \check \xi_i \le \epsilon$ for $i \in B$.
\end{lem}

\begin{proof} For (a), we note that the expected epoch length of $\textsf{Explore}(\{i\})$ is 
$$\mathbb E E_\ell = 1 + v_i \le 3.$$
Whenever $B \ne \emptyset$ is not empty, the procedure \textsf{EST-REDUCED} explores at least $\tau \ge \log(1 / \delta)$ epochs, so by Lemma \ref{lem:sumofepochlength}, with probability at least $1 - \delta$, the total number of time steps used by the procedure is 
\[
T \le 24 L \le 24 \cdot K \lvert B \rvert \tau. 
\]

For (b), we prove it by  applying the results in Lemma \ref{lem:expreq}. Note that we offer each item $i \in B$ for $K \tau \ge M \tau \ge (\frac{1}{4\nu_i} \land \frac{M}{2}) \tau$ epochs, so we meet the conditions in Lemma \ref{lem:expreq}, whose conclusion shows that (b) holds with probability at least $1 - 14 N \delta$. We apply a union bound to find that (a)(b) hold simultaneously with probability at least $1 - (\delta + 14 N \delta) \ge 1 - 15 N \delta \ge 1 - \delta_0$. 
\end{proof}

\section{Proofs for Sections \ref{sec:adaptive} and \ref{sec:peresult}}

\label{app:optimal}

\subsection{Error Analysis for Estimation of Advantage Score}

\label{app:erralyz}

We analyze the error of the estimations of $\nu_i, \srew$  when we use the generalized epoch-based offering procedure and how their error propagates to $\theta, \xi_i$.
By Proposition \ref{prop:ind} and Lemma \ref{lem:hpver2}, we know that the tail bound of $\nu_i $ satisfies 
\begin{align*}
    \hat \nu_i - \check \nu_i \lesssim \sqrt{\frac{\nu_i \iota}{T_i}}  + \frac{\iota}{T_i},
\end{align*}
where we use $\iota = \mathrm{polylog}(\delta^{-1}, N)$ to denote the polylogarithmic terms and $T_i $ is the number of epochs that item $i$ is offered. The major difference between this tail and the common $\frac{1}{\sqrt{T_i}}$-type tail bound (e.g. Lemma \ref{lem:hoeffding}) is the existence of the term $\sqrt{\nu_i}$. We fully exploit this term to show the exploration requirement (i.e. required number of epochs) of each item $i \in B$ in the following lemma. 

\begin{lem}[Exploration Requirement] \label{lem:expreq} For every item $i \in B$, if $T_i \ge T'_i \tau$, where $T'_i = (\frac{1}{4\nu_i} \land \frac{M}{2})$ and $\tau = O(\frac{\log(N / \delta_0)}{\epsilon^2})$ is as defined in Algorithm \ref{algo:refined}, then with probability at least $1 - 14 N \delta$, we have $\xi_i \in [\check \xi_i, \hat \xi_i]$ and $\hat \xi_i - \check \xi_i \le \epsilon$ for every $i \in B$. 
\end{lem}

Our focus is to show $\hat \xi_i - \check \xi_i \le \epsilon$, which requires us to combine the tail bound with the error propagation. In the following proof, we mainly analyze the tail bound itself and defer the error propagation analysis to Lemma \ref{lem:errprop2}. 

\begin{proof}[Proof of Lemma \ref{lem:expreq}]
For an item $i \in B$, by Lemma \ref{lem:hpver2}, with probability at least $1 - 13 \delta$,  we have $\nu_i \in [\check \nu_i, \hat \nu_i]$ and 
\begin{align} \notag
    \hat \nu_i - \check \nu_i &\le 2(\sqrt{\frac{196 \nu_i \logterm}{T_i}} + \frac{196 \logterm}{T_i}) \\ \notag
    &\le 2(\sqrt{\frac{196 \nu_i \logterm}{(\frac{1}{4\nu_b} \land \frac{M}{2}) \tau}} + \frac{196 \logterm}{(\frac{1}{4\nu_i} \land \frac{M}{2}) \tau}) \\ \notag
    &= 2(\sqrt{\frac{196 \nu_i \logterm}{(\frac{1}{4\nu_i} \land \frac{M}{2}) \frac{C_2 C_0  \logterm}{\epsilon^{2}}}} + \frac{196 \logterm}{(\frac{1}{4\nu_i} \land \frac{M}{2}) \frac{C_2 C_0  \logterm}{\epsilon^{2}}}) \\    \notag
    &= 2(\sqrt{\frac{\nu_i \epsilon^{2} }{(\frac{1}{4\nu_i} \land \frac{M}{2}) C_2}} + \frac{\epsilon^{2} }{(\frac{1}{4\nu_i} \land \frac{M}{2}) C_2}) \\ \notag
    &= 2(\sqrt{\frac{\nu_i (4\nu_i \lor \frac{2}{M})  }{ C_2}} \epsilon + \frac{ 4\nu_i \lor \frac{2}{M}}{ C_2}\epsilon^{2})  \\ \notag
    &\le 2(\sqrt{\frac{\nu_i (\nu_i \lor \frac{1}{M})  }{ C_2 / 4}} \epsilon + \frac{ \nu_i \lor \frac{1}{M}}{ C_2 / 4}\epsilon^{2} ) \\ \notag
    &\le 2(\frac{\nu_i \lor \frac{1}{M}}{\sqrt{ C_2 / 4}} \epsilon + \frac{ \nu_i \lor \frac{1}{M}}{ C_2 / 4}\epsilon^{2} ) \\ 
    &\le \frac{\nu_i \lor \frac{1}{M}}{\sqrt{ C_2 / 64}} \epsilon. \label{eq:alyz1}
\end{align}
By Lemma \ref{lem:hoeffding}, with probability at least $1 - \delta$, we have $\srew \in [\check \srew, \hat \srew]$ and 
\begin{align}
    \hat \srew - \check \srew \le 2\sqrt{\frac{\logterm}{2 T_Z}} \le   2\sqrt{\frac{\logterm}{2 \frac{C_2 C_0  \logterm}{\epsilon^{2}}}} = \frac{\epsilon}{\sqrt{C_2 C_0 /2}}. \label{eq:alyz2}
\end{align}
By a union bound, we have with probability at least $1 - (\delta + 13 \lvert B \rvert \delta) \ge 1 - 14 N \delta$ that $\nu_i \in [\check \nu_i, \hat \nu_i], \srew \in [\check \srew, \hat \srew]$ and  Eqs. (\ref{eq:alyz1})(\ref{eq:alyz2}) hold for $\nu_i $ and $\srew$ for all $i \in B$. When the event holds, we can use Lemma \ref{lem:mono} to show that $\xi_i \in [\check \xi_i, \hat \xi_i]$ for all $i \in B$ and use Lemma \ref{lem:errprop2} with $\epsilon_1 = \frac{\epsilon}{\sqrt{C_2 / 64}}, \epsilon_3 = \frac{\epsilon}{\sqrt{C_2 C_0 / 2}}$ to show that $\hat \xi_i - \check \xi_i \le \frac{4\epsilon}{\sqrt{C_2 / 64}} \le \epsilon$ for all $i \in B$.
\end{proof}

\begin{lem}[Error Propagation] \label{lem:errprop2}Assume $A \subseteq S^* \subseteq A \sqcup B$ and let $M = \min\{K - \lvert A \rvert, \lvert B \rvert \}$. Suppose we have $0 \le \hat \nu_i - \check \nu_i \le (\nu_i \lor \frac{1}{M}) \epsilon_1$ for every $i \in B$ and $0 \le \hat \srew - \check \srew \le \epsilon_3$.  Let $\check \theta, \hat \theta, \check \xi_i, \hat \xi_i$ be those defined in Eqs. (\ref{eq:esttheta}) (\ref{eq:estu}). Then $\hat \theta - \check \theta \le 2 \epsilon_1 + \epsilon_3$ and $\hat \xi_i - \check \xi_i \le 3 \epsilon_1 + \epsilon_3$.
\end{lem}

\begin{proof} Note that $\hat \nu_i \ge \check \nu_i$, so $\hat \nu_i - \check \nu_i \le (\hat \nu_i \lor \frac{1}{M}) \epsilon_1$. Using Lemma \ref{lem:errprop1}, we have 
$R(S, \hat \nu, \hat \srew) - R(S, \check \nu, \check \srew) \le 2\epsilon_1  + \epsilon_3.$ Note that $\hat \theta = R(S, \hat \nu, \hat \srew)$ and $\check \theta \ge R(S, \check \nu, \check \srew)$, together with Lemma \ref{lem:mono}, we prove $\hat \theta - \check \theta \le 2\epsilon_1  + \epsilon_3$.

For every $i \in B$, we have 
\[
      \hat{\xi}_i - \check \xi_i \le \lvert\hat{\nu}_i - \check{\nu}_i\rvert + \lvert \hat{\theta} - \check{\theta}\rvert \le (\nu_i \lor \frac{1}{M}) \epsilon_1 + 2\epsilon_1 + \epsilon_3 \le 3\epsilon_1 + \epsilon_3. \qedhere
      \]
\end{proof}

\begin{lem} \label{lem:errprop1} Suppose $\lvert S \rvert \le M$. Given $\srew, \srew'$ such that $0 \le \srew' \le \srew \le 1$ and $\nu_i, \nu'_i$ such that $0 \le  \nu'_i \le  \nu_i \le 1$ for every $i \in S$. Let $\epsilon_1, \epsilon_3 \in (0, 1]$. Suppose we have $ \nu_i - \nu'_i \le (\nu_i \lor \frac{1}{M})  \epsilon_1$ and $\srew - \srew' \le \epsilon_3$. Then we have $R(S, \nu, \srew) - R(S, \nu', \srew') \le 2\epsilon_1  + \epsilon_3$.
\end{lem}

\begin{proof} We have
    \begin{align*}
    R(S, \nu, \srew) - R(S, \nu', \srew') &= \frac{\srew + \sum_{i \in S} \nu_i r_i}{1 + \sum_{i \in S} \nu_i} - \frac{\srew' + \sum_{i \in S} \nu'_i r_i}{1 + \sum_{i \in S} \nu'_i} \\
    &\le \frac{(\srew - \srew') + \sum_{i \in S} (\nu_i - \nu'_i)}{1 + \sum_{i \in S} \nu_i} \\
    &\le {\epsilon_3} + \frac{\sum_{i \in S} (\nu_i + 1 / M) \epsilon_1}{1 + \sum_{i \in S} \nu_i} \\
    &\le {\epsilon_3} + 2 \epsilon_1. \qedhere
    \end{align*}

\end{proof}

\subsection{Proof of Lemma \ref{lem:coarse}}

\label{app:pecoarse}

\begin{proof}[Proof of Lemma \ref{lem:coarse}] For (a), in \textsf{EST-ROUGH}, we independently explore $L = N \tau = 4 N K \cdot 196 \logterm \ge 72 \logterm$ epochs with expected length $\mathbb E[E_\ell] = 1 + v_i \le 2$. By Lemma \ref{lem:sumofepochlength}, with probability at least $1 - 4 \delta$, the sample complexity is bounded by $T \le 5 L \lesssim N K \log \delta^{-1}$. For each $i \in [N]$, by Lemma \ref{lem:hpver2}, with probability at least $1 - 13 \delta$, we have $\tilde v_i = \hat \nu_i \ge v_i$ and 
\begin{align*}
    \tilde{v}_i - v_i \le \sqrt{\frac{196 v_i \log(2 / \delta)}{\tau}} + \frac{196 \log(2 / \delta)}{\tau} \le \sqrt{\frac{v_i}{4 K}} + \frac{1}{4 K} \le 2 v_i \lor \frac{1}{K}.
\end{align*}
Using a union bound, (a) holds with probability at least $1 - (13N + 4)\delta \ge 1 - 17N \delta \ge 1 -\delta_0$.

For (b), let $V = 1 + \sum_{i \in \sset} v_i$. We have 
\begin{align*}
    1 + \sum_{i \in \sset} \appv_i \in [1 + \sum_{i \in \sset} v_i, 1 + \sum_{i \in \sset} 2v_i + \frac{\lvert \sset \rvert}{K}] \subseteq [V, 2V].
\end{align*}

Therefore, we have 
\[
    \frac{\appv_i}{1 + \sum_{i \in \sset} \appv_i} \in [v_i, 2 v_i \lor \frac{1}{K}] / [V, 2V] \subseteq [\frac{v_i}{2V}, \frac{2v_i \lor \frac{1}{K}}{V}] \subseteq [\frac{\nu_i}{2}, 2 \nu_i \lor \frac{1}{K}]. \qedhere
\]
\end{proof}

\subsection{Proof of Lemma \ref{lem:refined}}

\label{app:perefined}

\begin{lem} \label{lem:refine1} At the end of \textsf{EST-ADAPTIVE}, for $b \in B$, we have $T_b \ge (\frac{1}{4\nu_b} \land \frac{M}{2}) \tau$.
\end{lem}

\begin{proof} Suppose $b\in B_i$. If $i < m$, we have $\appnu_b \in (\frac{1}{2 d_i}, \frac{1}{d_i}]$. By Lemma \ref{lem:coarse}, we have $(2\nu_b \lor \frac{1}{K}) \ge \appnu_b \ge \frac{1}{2d_i}$. Therefore, $d_i \ge \frac{1}{2(2\nu_b \lor \frac{1}{K})} = (\frac{1}{4\nu_b} \land \frac{K}{2}) \ge  (\frac{1}{4\nu_b} \land \frac{M}{2})$.

If $i = m$, we have $d_i = M \ge \frac{M}{2} \ge  (\frac{1}{4\nu_b} \land \frac{M}{2})$. We conclude by $T_b \ge d_i \tau$. 
\end{proof}

\begin{lem} \label{lem:refine2} With probability at least $1 - \delta$, \textsf{EST-ADAPTIVE} uses $T \le 120 \lvert B \rvert \tau$ time steps. 
\end{lem}

\begin{proof} Let $L$ be the total number of epochs. Note that for every $B_{i, j}$, the expected epoch length of $\textsf{Explore}(B_{i,j})$ is
$$\mathbb E E_\ell = 1 + \sum_{b \in B_{i, j}} \nu_b \le 1 + \sum_{b \in B_{i, j}} 2\appnu_b \le 1 + \lvert B_{i, j} \rvert \cdot 2 \cdot 2^{-i} \le 1 + d_i \cdot 2 \cdot 2^{-i} \le 1 + 2 = 3.$$

The total number of epochs is 
\begin{align*}
    L &= \tau \cdot \sum_{i = 0}^m \sum_{j = 1}^{c_i} d_i \\
    &\le \tau \cdot \sum_{i = 0}^m (\sum_{j = 1}^{c_{i - 1}} \lvert B_{i, j} \rvert + d_i) \\
    &\le \tau \cdot (\sum_{i = 0}^m \sum_{j = 1}^{c_{i - 1}} \lvert B_{i, j} \rvert + \sum_{i = 0}^m d_i) \\
    &\le \tau \cdot (\lvert B \rvert + 2^{m + 1}) \\
    &\le \tau \cdot (\lvert B \rvert + 4M) \\
    &\le \tau \cdot 5\lvert B \rvert.
\end{align*}
Assume $B \ne \emptyset$. Then $L \ge \tau \ge \log(1 / \delta)$. By Lemma \ref{lem:sumofepochlength}, with probability at least $1 - \delta$, we have $T \le 24 L\le 120 \lvert B\rvert \tau$.
\end{proof}

\begin{proof}[Proof of Lemma \ref{lem:refined}] By Lemma \ref{lem:refine1}, we meet the exploration requirement in Lemma \ref{lem:expreq}. Using a union bound, we find that Lemmas \ref{lem:expreq} and \ref{lem:refine2} hold simultaneously with probability at least $1 - (\delta + 14 N \delta) \ge 1 - \delta_0$. Note that Lemma \ref{lem:refine2} implies (a) and Lemma \ref{lem:expreq} implies (b).
\end{proof}

\subsection{Proofs of Theorems \ref{thm:pe} and \ref{thm:pac}}

\label{app:peresult}

\begin{proof}[Proof of Theorem \ref{thm:pe}] By Lemma \ref{lem:coarse}, \textsf{EST-ROUGH} gives a rough estimation of $v_i$ with probability at least $1 - \frac{\delta}{2}$. Given those rough estimations, by Lemmas \ref{lem:peframework} and \ref{lem:refined}, \textsf{SAR-MNL} with \textsf{EST-ADAPTIVE} is $\frac{\delta}{2}$-PAC. So the proposed algorithm returns optimal assortment with probability at least $(1 - \frac{\delta}{2})^2 \ge 1 - \delta$ and thus it is $\delta$-PAC. We conclude by noting that we have $C_{\textsf{EST}} = O(1)$ in Lemma \ref{lem:enhancedpeframework} for \textsf{EST-ADAPTIVE}.
\end{proof}

\begin{proof}[Proof of Theorem \ref{thm:pac}] We stop the algorithm provided in the proof of Theorem \ref{thm:pe} at the phase $k$ when $\epsilon_{k-1} \le \frac{\varepsilon}{3}$. Then we return the assortment $S$ corresponding to $\hat \theta$. Specifically, we return $S = A^{(k-1)} \sqcup S_0$,
$$S_0 = \argmax_{S_0 \subseteq B^{(k - 1)}: \lvert S_0 \rvert \le M} R(S_0, \hat \nu, \hat \srew).$$

Following the proof of Lemma \ref{lem:enhancedpeframework}, we can show the desired sample complexity bound. Moreover, the returned assortment satisfies
\begin{align*}
    \theta^* - R(S, \bm v) &\le \hat \theta - R(S, \bm v) \\
    &= R(S_0, \hat \nu, \hat \srew) - R(S_0, \nu, \srew) \\
    (\text{Lemma \ref{lem:errprop1}}) &\le 3 \epsilon_{k-1} \\
    &\le \varepsilon. \qedhere
\end{align*}
\end{proof}

\section{Proofs for Section \ref{sec:sarreg}}
\label{app:reg}

Our algorithm is to invoke \textsf{SAR-MNL} with $\delta = \frac{1}{T}$ and the procedure \textsf{EST-REG}. Note that this algorithm could possibly return the optimal assortment $S^*$ before the time horizon $T$ is reached. In this case, we assume our algorithm keeps offering $S^*$ until reaching the time horizon. Note that offering $S^*$ incurs zero regret.

We use $\check \xi_i^{(k)}, \hat \xi_i^{(k)}$ for $i \in B^{(k-1)}$ to denote the values $\{\check \xi_i, \hat \xi_i\}_{i \in B^{(k-1)}}$ returned by \textsf{EST-REG} in phase $k$. We assume $\check \xi_i^{(0)} = 0$ and $\hat \xi_i^{(0)} = 1$. The following lemma summarizes the important guarantees of \textsf{SAR-MNL} that we need to show the regret bound.

\begin{lem} \label{lem:estreg} With probability at least $1 - \frac{1}{T}$, throughout the algorithm, we have that $u_i \in [\check \xi_i^{(k)}, \hat \xi_i^{(k)}]$, $\hat \xi_i^{(k)} - \check \xi_i^{(k)} \le \frac{\epsilon_k}{2}$, $A^{(k)} \subseteq S^* \subseteq A^{(k)} \sqcup B^{(k)}$, and $B^{(k)} \subseteq \{i \in [N] : \Delta_i \le \epsilon_k\}$ for every phase $k$.
\end{lem}

\begin{proof} We claim \textsf{EST-REG} satisfies the condition (b) in Lemma \ref{lem:peframework}. Then we can follow the proof of Lemma \ref{lem:peframework} to show the  that with probability at least $1  -\frac{1}{T}$, we have  $u_i \in [\check \xi_i^{(k)}, \hat \xi_i^{(k)}]$, $\hat \xi_i^{(k)} - \check \xi_i^{(k)} \le \frac{\epsilon_k}{2}$, $A^{(k)} \subseteq S^* \subseteq A^{(k)} \sqcup B^{(k)}$, and $B^{(k)} \subseteq \{i \in [N] : \Delta_i \le \epsilon_k\}$ throughout the algorithm. 

To show \textsf{EST-REG} satisfies (b) in Lemma \ref{lem:peframework}, we need to analyze the  error of the estimations it returns. Note that \textsf{EST-REG} offers each item $i$ for $T_i \ge K \tau$ epochs, which satisfies the exploration requirement in Lemma \ref{lem:expreq}. Therefore, it returns  $u_i \in [\check \xi_i^{(k)}, \hat \xi_i^{(k)}]$, $\hat \xi_i^{(k)} - \check \xi_i^{(k)} \le \frac{\epsilon_k}{2}$ with the desired probability. Thus it satisfies (b) in Lemma \ref{lem:peframework}.
\end{proof}

Now we start to analyze the regret. The key observation is that Lemma \ref{lem:comparison} enables us to represent the regret of offering an assortment $S$ in terms of the score difference between $S$ and $S^*$. Specifically, when $\sset = \emptyset$, the regret of $\textsf{Explore}(S)$ is $\sum_{i \in S^* \setminus S} u_i - \sum_{i \in S \setminus S^*} u_i.$ Therefore, if we know that $A^{(k)} \subseteq S^* \subseteq A^{(k)} \sqcup B^{(k)}$ and we choose a maximum subset $B \subseteq B^{(k)}$ to construct an assortment $S = A \sqcup B$ such that $\lvert S \rvert = K$, then the regret of $\textsf{Explore}(S)$ is bounded by 
\begin{align}
\lvert B \rvert (\max_{i \in B} u_i) - \lvert B^* \rvert (\min_{i \in B^*} u_i) \le (K - \lvert A \rvert) (\max_{i \in B^{(k)}} u_i - \min_{i \in B^{(k)}} u_i). \label{eq:reg3}
\end{align}
In the following, Lemma \ref{lem:reg1} bounds the right hand side of Eq. (\ref{eq:reg3}), based on which Lemma \ref{lem:reg2} bounds the regret of \textsf{EST-REG}.

\begin{lem} \label{lem:reg1}We have  $(\max_{i \in B^{(k)}} \hat \xi_i^{(k)}) - (\min_{i \in B^{(k)}} \check \xi_i^{(k)}) \le \frac{3}{2}\epsilon_k$ and $(\max_{i \in B^{(k)}} \hat \xi_i^{(k)}) \ge 0$. 
\end{lem}

\begin{proof} The second statement $(\max_{i \in B^{(k)}} \hat \xi_i^{(k)}) \ge 0$ follows directly from that \textsf{EST-REG} rejects items with negative scores. Next we show the first statement. We write $\check \xi_i = \check \xi_i^{(k)}, \hat \xi_i =  \hat \xi_i^{(k)}$ and $\epsilon = \frac{\epsilon_k}{2}$. If $\lvert B^{(k-1)} \rvert \le M$, we have $\check \xi_i \le 0 \le \hat \xi_i$ for $i \in B^{(k)}$ by the definitions of $B_{\mathrm{acc}}, B_{\mathrm{rej}}$ and that $B_{\mathrm{acc}}, B_{\mathrm{rej}}$ are excluded from $B^{(k)}$. We conclude by
$$\hat \xi_i - \check \xi_j \le \hat \xi_i - \check \xi_i + \hat \xi_j - \check \xi_j \le 2 \epsilon.$$

If $\lvert B^{(k-1)} \rvert > M$, then we have $B_{\mathrm{acc}} = \{b \in B^{(k-1)} : \check \xi_b > (0 \lor \beta)\}$ and $B_{\mathrm{rej}} = \{b \in B^{(k-1)} : \hat \xi_b < (0 \lor \alpha)\}$. Therefore, we have $\check \xi_i \le (0 \lor \beta)$ and $\hat \xi_i \ge (0 \lor \alpha)$ for $i \in B^{(k)}$. For each $i, j \in B^{(k)}$, if $\check \xi_i \le 0$, then we have $\hat \xi_i - \check \xi_j \le 2 \epsilon$ using previous equation. Otherwise, we have 
\begin{align*}
    \hat \xi_i - \check \xi_j \le \beta + \epsilon - \check \xi_j \le \beta + 2\epsilon - \hat \xi_j \le 2\epsilon + \beta - (0 \lor \alpha). 
\end{align*}
It suffices to show $\beta - (0 \lor \alpha) \le \epsilon$. Assume $\beta \ge 0$. Next we show $\beta - \alpha \le \epsilon$. Let $\hat \xi_{i_1}, \ldots, \hat \xi_{i_M}$ be the $M$ largest values of $\{\hat \xi_i\}_{i\in B^{(k-1)}}$. By the definition of $\alpha$, we have $\alpha \ge \min_{1 \le j \le M} \check \xi_{i_j}$. Suppose $\alpha \ge \check \xi_{i_x}$ for $x \in [M]$. We have $\beta - \alpha \le \hat \xi_{i_x} - \check \xi_{i_x} \le \epsilon$. 
\end{proof}

\begin{lem} \label{lem:reg2} The regret incurred by \textsf{EST-REG} in phase $k$ is $\mathrm{Reg}^{(k)} \lesssim \lvert B \setminus S^* \rvert \cdot \frac{K \log NT}{\epsilon_k}$.
\end{lem}

\begin{proof} We note that $B = B^{(k-1)}$ and $A = A^{(k-1)}$. Let $B^* = S^* \setminus A$. For every $j \in [m]$, let $S'_j = A^{(k-1)} \sqcup B'_j$ and $B^*_j = B^* \cap B'_j$.  Note that $\lvert B'_j \setminus B^*_j \rvert = M \ge \lvert B^* \setminus B^*_j \rvert$. Since $\sset = \emptyset$, by Proposition \ref{prop:ind} and Lemma \ref{lem:comparison}, the regret incurred by $\textsf{Explore}(S'_j)$ is 
\begin{align*}
    (1 + \sum_{i \in S'_j} v_i) (R(S^*, \bm v) - R(S'_j, \bm v)) &= \sum_{i \in S^*} u_i - \sum_{i \in S'_j} u_i \\
    &= \sum_{i \in B^* \setminus B^*_j} u_i - \sum_{i \in B'_j \setminus B^*_j} u_i \\
    &\le \lvert B^* \setminus B^*_j \rvert \cdot \max_{i \in B} \{u_i \lor 0\} - \lvert B'_j \setminus B^*_j \rvert \cdot \min_{i \in B} u_i \\
   (\text{Lemma \ref{lem:reg1}})
    &\le \frac{3}{2} \epsilon_{k-1} \lvert B'_j \setminus B^*_j \rvert \\
    &= 3 \epsilon_{k} \lvert B'_j \setminus B^*_j \rvert.
\end{align*}
Note that for every $b \in B^{(k-1)} \setminus B^*$, there are at most two $j \in [m]$ such that $b \in B_j' \setminus B_j^*$, so we have 
\begin{align}
    \sum_{j = 1}^m \lvert B'_j \setminus B^*_j \rvert \le 2 \lvert B^{(k-1)} \setminus B^* \rvert. \label{eq:fullasst}
\end{align} Thus the regret incurred in phase $k$ is 
\[
    \mathrm{Reg}^{(k)} \le K \tau  \sum_{j = 1}^m (1 + \sum_{i \in S'_j} v_i) (R(S^*, \bm v) - R(S'_j, \bm v)) 
    \le  K\tau  \sum_{j = 1}^m \lvert B'_j \setminus B^*_j \rvert 
    \le 6\epsilon_k \tau  \lvert B^{(k-1)} \setminus B^* \rvert. \qedhere
    \]
\end{proof}

\begin{proof}[Proof of Theorem \ref{thm:sereg}]
Since that the event specified in Lemma \ref{lem:estreg} happens with probability $1 - \frac{1}{T}$ and that the regret is bounded by $\mathrm{Reg}_T \le T$, it suffices we prove the regret bound under the event, which is 
  \begin{align*}
      \mathrm{Reg}_T &= \sum_{k = 1}^{T} \mathrm{Reg^{(k)}} \\
      (\text{Lemma \ref{lem:reg2}})&\lesssim \sum_{k = 1}^{\infty} \lvert B^{(k-1)} \setminus S^* \rvert \cdot \frac{K \log(NT)}{\epsilon_k} \\
      &= \sum_{k = 1}^{\infty} \frac{K \log(NT)}{\epsilon_k} \cdot \sum_{i \in [N]\setminus S^*} \mathbb I\{ i \in B^{(k-1)}\} \\
      (\text{Lemma \ref{lem:estreg}})&\le \sum_{k = 1}^{\infty} K \log(NT) \cdot \sum_{i \in [N]\setminus S^*} \frac{\mathbb I\{\Delta_i \le \epsilon_{k-1}\}}{\epsilon_k} \\
      &=\sum_{i \in [N]\setminus S^*} K \log(NT) \cdot \sum_{k = 1}^{\infty}  \frac{\mathbb I\{\Delta_i \le \epsilon_{k-1}\}}{\epsilon_k} \\
      &\lesssim \sum_{i \in [N] \setminus S^*} \frac{ K \log NT }{\gap_i}. \qedhere
  \end{align*}
\end{proof}

Finally, we discuss why we always offer full assortments in \textsf{EST-REG}. Actually, this is utilized by Eq. (\ref{eq:fullasst}). If we do not offer the full assortments, the right hand side of Eq. (\ref{eq:fullasst}) could become $2 \lvert B^{(k-1)} \rvert$. Thus we could end up with a regret bound that depends on $S^*$, as we show in Section \ref{sec:sarreg}.

\section{Lower Bounds}

\label{app:lb}

We recall the definition of $P_S^{\bm v}$, the probability distribution of assortment $S$ under MNL choice model with preference parameter $\bm v$. 
\begin{align}
    P_S^{\bm v}(i) = \begin{cases} \frac{v_i}{v_0 + \sum_{j \in S} v_j}, & i \in S \cup \{0\}, \\
    0 & \text{otherwise}. 
    \end{cases} \label{eq:mnl}
\end{align}

We show the following lower bound under the restriction $\Delta_i \lesssim \frac{1}{K}$, which gives us enough freedom to construct a simple MNL-bandit instance to realize it, as in Lemma \ref{lem:lbinst}. Note that our regret upper bound in Theorem \ref{thm:sereg} only depends on  items in $[N] \setminus S^*$, so in our lower bound, we only consider the gap sequence of items in $[N] \setminus S^*$. We highlight that our lower bound is for every $K$. 

\begin{thm} \label{thm:lb} Suppose an algorithm $\mathcal A$ achieves $\mathbb E[\mathrm{Reg}_T] \lesssim T^p$ on any MNL-bandit instance for a constant $p \in (0, 1)$. For any $N \ge 2, K \le \frac{N}{2}$, suboptimality gap sequence $\{\gap_i\}_{i = K + 1}^N$ such that $\max_{i} \gap_i \le \frac{1}{16 K}$, there is a MNL-bandit instance $\mathcal I$ that realizes the gap sequence. Moreover, for this instance, we have $S^* = [K]$ and the algorithm incurs regret  
$$\liminf_{T \to \infty} \frac{\mathrm{Reg}_T}{\log T} \gtrsim \sum_{i \in [N] \setminus S^*} \frac{1}{K \gap_i}.$$
\end{thm}

For any assortment $S\subseteq [N]$ with $\lvert S \rvert \le K$, let $\mathcal T_S(T)$ be the number of time steps that $S$ is offered. For any item $i \in [N]$, let $\mathcal T_i(T) = \sum_{\lvert S \rvert \le K: i \in S} \mathcal T_S(T)$ be the number of time steps that item $i$ is offered. Next we prove Theorem \ref{thm:lb}. Our proof is inspired by the proofs of the similar lower bounds in multi-armed bandits \citep{lattimore2020}. 

\begin{lem}[Bretagnolle-Huber inequality] \label{lem:bhineq}Let $\mathbb P, \mathbb P'$ be two measures over the same measurable space. Let $A$ be an event. Then 
$$\mathbb P (A) + \mathbb Q (A^\complement) \ge \frac{1}{2} \exp(-D_{\mathrm{KL}}(\mathbb P \parallel \mathbb Q)),$$
where $D_{\mathrm{KL}}(\cdot \parallel \cdot)$ is the Kullback–Leibler divergence between probability measures. 
\end{lem}

\begin{lem} \label{lem:lbinst} Assume the conditions of Theorem $\ref{thm:lb}$. For every $i \in [N]$, we let $r_i = 1$ and 
\begin{align*}
    v_i = \begin{cases}
        \frac{1}{K} + \frac{1}{2K(K-1)}, & i < K, \\
        \frac{1}{2K}, & i = K, \\
        \frac{1}{2K} - \frac{4\gap_i}{1 + 2 \gap_i}, & i > K. 
    \end{cases}
\end{align*}

Then $\mathcal I = (N, K, \bm r, \bm v)$ is a MNL-bandit instance in which $\gap_i$ complies with Definition \ref{defn:sub}. 
\end{lem}

\begin{proof} Note that for $K = 1$ we have $\max_{i \in [N]} v_i \le \frac{1}{2K} \le 1$ and for $K \ge 2$ we have $\max_{i \in [N]} v_i \le \frac{1}{K} + \frac{1}{2K(K-1)} \le \frac{1}{2} + \frac{1}{4} \le 1$, so we always have $\max_{i \in [N]} v_i \le 1$. Note that by the assumption $\gap_i \le \frac{1}{16K}$ we have $\frac{4\gap_i}{1 + 2 \gap_i} = \frac{4}{1 / \gap_i + 2} \le \frac{4}{16K + 2} \le \frac{1}{4K}$, so we have $\min_{i \in [N]} v_i \ge\frac{1}{4K} > 0$. Since $v_i, r_i \in [0, 1]$, we know that $\mathcal I$ defines a MNL-bandit instance. 

Let $S^* = \{1, 2, \ldots, K\}$ be the optimal assortment in this instance. For every item $i \in [N]\setminus S^*$, let $S^*_i = \arg \max_{\lvert S \rvert \le K: i \in S} R(S, \bm v)$ be the best assortment containing $i$. We next show $\gap_i = R(S^*, \bm v) - R(S^*_i, \bm v)$. By direct computations, we have $S^*_i = \{1, 2, \ldots, K - 1, i\}$ for $i \notin S^*$ and $S^*_i = S^*$ for $i \in S^*$. Therefore, we have $R(S^*, \bm v) - R(S^*_i, \bm v) = 0 = \gap_i$ for $i \in S^*$. Note that $\sum_{i = 1}^K v_i = 1$, so $R(S^*, \bm v) = \frac{1}{2}$. For $i \notin S^*$ we have 
\begin{align*}
    R(S^*, \bm v) - R(S^*_i, \bm v) &= \frac{1}{2} - \frac{\sum_{i = 1}^K v_i - \frac{4\gap_i}{1 + 2 \gap_i}}{1 + \sum_{i = 1}^K v_i - \frac{4\gap_i}{1 + 2 \gap_i}} \\
    &= \frac{1}{2} - \frac{1 - \frac{4\gap_i}{1 + 2 \gap_i}}{2 - \frac{4\gap_i}{1 + 2 \gap_i}} \\
    &= \gap_i. \qedhere
\end{align*}
\end{proof}

\begin{lem} \label{lem:lbregitem} Under the MNL-bandit instance $\mathcal I$ defined in Lemma \ref{lem:lbinst}, we have 
\begin{align*}
\mathrm{Reg}_T \ge\frac{1}{2}  \sum_{i \in [N] \setminus S^*} \mathbb E[\mathcal T_i(T)] \cdot \gap_i.
\end{align*}
\end{lem}

\begin{proof} Under instance $\mathcal I$, for any assortment $S \subseteq [N]$ with $\lvert S \rvert \le K$, let $B = S \setminus S^*$, we have 
\begin{align*}
    \theta^* - R(S, \bm v) &= \frac{1}{2} - \frac{\sum_{i \in S} v_i}{1 + \sum_{i \in S} v_i} \\
    &\ge \frac{1}{2} - \frac{1 - \sum_{i \in B} \frac{4\gap_i}{1 + 2 \gap_i}}{2 - \sum_{i \in B} \frac{4\gap_i}{1 + 2 \gap_i}} \\
    &= \frac{\sum_{i \in B} \frac{4\gap_i}{1 + 2 \gap_i}}{2 (2 - \sum_{i \in B} \frac{4\gap_i}{1 + 2 \gap_i})} \\
    &\ge \sum_{i \in B} \frac{\gap_i}{1 + 2 \gap_i} \\
    &\ge \sum_{i \in B} \frac{\gap_i}{2},
\end{align*}
where in the second-to-third inequality we used $\frac{\alpha}{1 + \alpha} \le \frac{\alpha + \beta}{1 + \alpha + \beta}$ for $\alpha, \beta > 0$ with $\alpha = \sum_{i \in S} v_i$ and $\beta = \sum_{i \in S^*} v_i - \sum_{i \in S} \max\{v_i, \frac{1}{2K}\}$, and in the fifth-to-last inequality we used $\gap_i \le \frac{1}{16K} \le \frac{1}{2}$.

Recall that $S_t$ is the assortment offered at time step $t$. We have 
\begin{align*}
    \mathbb E[\mathrm{Reg}_T] &= \sum_{t = 1}^T \mathbb E[\theta^* - R(S_t, \bm v)] \\
    &= \sum_{t = 1}^T \sum_{\lvert S \rvert \le K} \mathbb E[\mathbb I \{S_t = S\}] (\theta^* - R(S, \bm v)) \\
    &= \sum_{\lvert S \rvert \le K} (\theta^* - R(S, \bm v)) \cdot \left(\sum_{t = 1}^T \mathbb E[\mathbb I \{S_t = S\}]\right) \\
    &= \sum_{\lvert S \rvert \le K} \mathbb E[\mathcal T_S(T)] \cdot (\theta^* - R(S, \bm v)) \\
    &\ge \sum_{\lvert S \rvert \le K} \mathbb E[\mathcal T_S(T)] \cdot \sum_{i \in S \setminus S^*} \frac{\gap_i}{2} \\
    &= \sum_{i \in [N] \setminus S^*} \mathbb E[\mathcal T_i(T)] \cdot \frac{\gap_i}{2}. \qedhere
\end{align*}

\end{proof}

\begin{lem} \label{lem:kl} Let $S \subseteq [N]$ with $\lvert S \rvert \le K$ be an assortment. Let $\bm v, \bm v'$ be two preference vectors such that $v'_x \ge v_x$ and $v'_i = v_i$ for $i \ne x$.  Then 
\begin{align*}
    D_{\mathrm{KL}}(P_S^{\bm v} \parallel P_S^{\bm v'}) \le \frac{(v'_x - v_x)^2}{2 v_x (1 + \sum_{i \in S} v_i)}.
\end{align*}
\end{lem}

\begin{proof} Recall the definition of $P_S^{\bm v}$ in Eq. (\ref{eq:mnl}). Let $P = P_S^{\bm v}$ and $Q = P^{\bm v'}_S$. We have 
\begin{align*}
    D_{\mathrm{KL}}(P_S^{\bm v} \parallel P_S^{\bm v'}) &= \sum_{i \in S \cup \{0\}} \frac{v_i}{1 + \sum_{j \in S} v_j} \cdot \log \frac{v_i / (1 + \sum_{j \in S} v_j)}{v'_i / (1 + \sum_{j \in S} v'_j)} \\
    &= \sum_{\substack{i \in S \cup \{0\} \\ i \ne x}} \frac{v_i}{1 + \sum_{j \in S} v_j} \cdot \log \frac{1 + \sum_{j \in S} v_j}{1 + \sum_{j \in S} v'_j} + \frac{v_x}{1 + \sum_{j \in S} v_j} \cdot \log \frac{v_\ell / (1 + \sum_{j \in S} v_j)}{v'_x / (1 + \sum_{j \in S} v'_j)} \\ 
    &= \log \frac{1 + \sum_{j \in S} v'_j}{1 + \sum_{j \in S} v_j} + \frac{v_x}{1 + \sum_{j \in S} v_j} \log \frac{v_x}{v'_x}.
\end{align*}
Denote $\delta = v_x' - v_x \ge 0$ and $V = 1 + \sum_{j \in S} v_j$. We have 
\begin{align*}
     D_{\mathrm{KL}}(P_S^{\bm v} \parallel P_S^{\bm v'}) &= \log(1 + \frac{\delta}{1 + \sum_{j \in S} v_j}) - \frac{v_x}{1 + \sum_{j \in S} v_j} \log(1 + \frac{\delta}{v_x}) \\
    &\le \frac{\delta}{V} - \frac{v_x}{V} (\frac{\delta}{v_x} - \frac{\delta^2}{2 v_x^2}) \\
    &\le \frac{\delta^2}{2 v_x V},
\end{align*}
where we used Taylor's formula $x - \frac{x^2}{2} \le \log(1 + x) \le x$ in the second-to-third inequality.
\end{proof}

\begin{lem} \label{lem:kldecomp} Let $\mathcal I = (N, K, \bm r, \bm v), \mathcal I' = (N, K, \bm r, \bm v')$ be two MNL-bandit instances and $\mathcal A$ be an algorithm. Let $\mathbb P$ be the probability measure induced by $\mathcal A$ and $\mathcal I$ and $\mathbb P'$ be that by $\mathcal A$ and $\mathcal I'$. We have
\begin{align*}
    D_{\mathrm{KL}}(\mathbb P \parallel \mathbb P') = \sum_{\lvert S \rvert \le K} \mathbb E[\mathcal T_S(T)] D_{\mathrm{KL}}(P_S^{\bm v} \parallel P_S^{\bm v'}).
\end{align*}
\end{lem}

\begin{proof} The lemma can be proved by following the proof of Lemma 15.1 in \citep{lattimore2020}.
\end{proof}

\begin{lem} \label{lem:lbitem} Under the assumptions of Theorem \ref{thm:lb} and the MNL-bandit instance defined in Lemma \ref{lem:lbinst}, for algorithm $\mathcal A$ and any item $i \in [N] \setminus S^*$, we have 
$$\liminf_{T \to \infty} \frac{\mathbb E[\mathcal T_i(T)]}{\log T} \ge \frac{1 - p}{32 K \gap_i^2}.$$
\end{lem}

\begin{proof} Fix an item $i \in [N] \setminus S^*$. For instance $\mathcal I$, we have 
\begin{align}
    \mathbb E[\mathrm{Reg}_T] &= \sum_{t = 1}^T \mathbb E[\theta^* - R(S_t, \bm v)]) \notag \\ 
    &\ge \sum_{t = 1}^T \mathbb E[\mathbb I\{i \in S_t \} \cdot (\theta^* - R(S_t, \bm v))] \notag \\
    &\ge \sum_{t = 1}^T \mathbb E[\mathbb I\{i \in S_t \} \cdot \gap_i] \notag \\
    &= \mathbb E[\mathcal T_i(T)] \cdot \gap_i.\label{eq:lb3}
\end{align}

We construct another MNL-bandit instance $\mathcal I'$. Let $\epsilon \in (0, \frac{1}{2K})$ be a parameter. We define a preference vector $\bm v'$ such that 
\begin{align*}
    v'_j = \begin{cases} 
    v_j, & j\ne i, \\
    \frac{1}{2K} + \epsilon, & j = i.
    \end{cases}
\end{align*}
Then $\mathcal I' = (N, K, \bm r, \bm v')$ is an MNL-bandit instance. 
For any algorithm $\mathcal A$, let $\mathbb P$ be the probability measure given by $\mathcal A$ and $\mathcal I$, and $\mathbb P'$ be that given by $\mathcal A$ and $\mathcal I'$. From now on, we use $\mathbb E$ to denote the expectation under $\mathbb P$, and $\mathbb E'$ to denote that under $\mathbb P'$. 

For instance $\mathcal I'$, direct computations give that 
\begin{align*}
    \max_{\lvert S \rvert \le K} R(S, \bm v') &= R(\{1, \ldots, K-1, i\}, \bm v') = \frac{1 + \epsilon}{2 + \epsilon}, \\
    \max_{\lvert S \rvert \le K : i\notin S} R(S, \bm v') &= R(\{1, \ldots, K\}, \bm v') = \frac{1}{2}. 
\end{align*}
Thus we have 
\begin{align}
    \mathbb E'[\mathrm{Reg}_T] &= \sum_{t = 1}^T \mathbb E[\max_{\lvert S \rvert \le K} R(S, \bm v') - R(S_t, \bm v)] \notag \\
    &\ge \sum_{t = 1}^T \mathbb E'[\mathbb I\{i \notin S_t \} \cdot (\max_{\lvert S \rvert \le K} R(S, \bm v') - R(S_t, \bm v'))] \notag \\
    &\ge \sum_{t = 1}^T \mathbb E'[\mathbb I\{i \notin S_t \}] \cdot (\frac{1 + \epsilon}{2 + \epsilon} - \frac{1}{2}) \notag \\
    &\ge \sum_{t = 1}^T \mathbb E'[\mathbb I\{i \notin S_t \}] \cdot \frac{\epsilon}{4}. \label{eq:lb4}
\end{align}

Recall that in the proof of Lemma \ref{lem:lbinst}, we showed $v_i \ge \frac{1}{4K}$. For any assortment $S$, by Lemma \ref{lem:kl}, we have
\begin{align*}
    D_{\mathrm{KL}}(P_S^{\bm v} \parallel P_S^{\bm v'}) &\le \frac{(\epsilon + \frac{4 \gap_i}{1 + 2 \gap_i})^2}{2 v_i (1 + \sum_{i \in S} v_i) } \\ 
    &\le \frac{(\epsilon + 4 \gap_i)^2}{2 v_i } \\ 
    &\le 2K (4 \gap_i + \epsilon)^2.
\end{align*}

By Lemma \ref{lem:kldecomp}, we have 
\begin{align*}
    D_{\mathrm{KL}}(\mathbb P \parallel \mathbb P') &= \sum_{\lvert S \rvert \le K} \mathbb E[\mathcal T_S(T)] D_{\mathrm{KL}}(P_S^{\bm v} \parallel P_S^{\bm v'}) \\
    &= \sum_{\lvert S \rvert \le K: i \in S} \mathbb E[\mathcal T_S(T)] D_{\mathrm{KL}}(P_S^{\bm v} \parallel P_S^{\bm v'}) \\
    &\le \sum_{\lvert S \rvert \le K: i \in S} \mathbb E[\mathcal T_S(T)] \cdot  2K (4 \gap_i + \epsilon)^2  \\
    &= 2K (4 \gap_i + \epsilon)^2 \cdot \mathbb E[\mathcal T_i(T)].
\end{align*}

Let $A = \{\mathcal T_i(T) > \frac{T}{2}\}$ be an event. By Lemma \ref{lem:bhineq}, we have 
\begin{align*}
    \mathbb P(A) + \mathbb P' (A^\complement) &\ge \frac{1}{2}\exp(-D_{\mathrm{KL}}(\mathbb P \parallel \mathbb P')) \\
    &\ge \frac{1}{2}\exp(-2K (4 \gap_i + \epsilon)^2 \cdot \mathbb E[\mathcal T_i(T)]). 
\end{align*}

By Markov's inequality, we have $\mathbb E[\mathcal T_i(T)] \ge \mathbb P(A) \cdot \frac{T}{2}$ and $\sum_{t = 1}^T \mathbb E'[\mathbb I\{i \notin S_t \}] \ge \mathbb P'(A^\complement) \cdot \frac{T}{2}$. Together with Eqs. (\ref{eq:lb3}) (\ref{eq:lb4}), we have 
\begin{align*}
    \mathbb E[\mathrm{Reg}_T] + \mathbb E'[\mathrm{Reg}_T] &\ge \mathbb E[\mathcal T_i(T)] \cdot \gap_i + \sum_{t = 1}^T \mathbb E'[\mathbb I\{i \notin S_t \}] \cdot \frac{\epsilon}{4} \\ 
    &\ge \frac{T}{2} (\mathbb P(A) \gap_i + \mathbb P'(A^\complement) \cdot \frac{\epsilon}{4}) \\
    &\ge \frac{T}{2} \min\{\gap_i, \frac{\epsilon}{4}\} (\mathbb P(A) + \mathbb P' (A^\complement)) \\
    &\ge \frac{T}{2} \min\{\gap_i, \frac{\epsilon}{4}\} \exp(-\mathbb E[\mathcal T_i(T)] \cdot 2K (4 \gap_i + \epsilon)^2 ) \\
\end{align*}

Recall that $\mathbb E[\mathrm{Reg}_T] + \mathbb E'[\mathrm{Reg}_T] \le 2 T^p$ for some $p \in (0, 1)$. As a result, we have 
\begin{align*}
    \liminf_{T \to \infty} \frac{\mathbb E[\mathcal T_i(T)]}{\log T} &\ge \frac{1}{2K (4 \gap_i + \epsilon)^2}(1-p - \limsup_{T\to\infty} \frac{\log(\frac{4}{\min\{\gap_i, \frac{\epsilon}{4}\}})}{\log T}) \\
    &=\frac{1 - p}{2K (4 \gap_i + \epsilon)^2}. 
\end{align*}
Let $\epsilon\to 0$, we have 
\[
    \liminf_{T \to \infty} \frac{\mathbb E[\mathcal T_i(T)]}{\log T} \ge \frac{1 - p}{32K \gap_i^2}. \qedhere
\]

\end{proof}

\begin{proof}[Proof of Theorem \ref{thm:lb}] We consider the MNL-bandit instance defined in Lemma \ref{lem:lbinst}. By Lemmas \ref{lem:lbregitem} and \ref{lem:lbitem}, we have 
\begin{align*}
    \liminf_{T \to \infty} \frac{\mathbb E[\mathrm{Reg}_T]}{\log T} &\ge \frac{1}{3} \sum_{i \in [N]\setminus S^*} \liminf_{T \to \infty} \frac{\mathbb E[\mathcal T_i(T)]}{\log T} \\
    &\ge \sum_{i \in [N]\setminus S^*} \frac{1-p}{96 K \gap_i}. \qedhere
\end{align*}
\end{proof}

\begin{thebibliography}{32}
\providecommand{\natexlab}[1]{#1}
\providecommand{\url}[1]{\texttt{#1}}
\expandafter\ifx\csname urlstyle\endcsname\relax
  \providecommand{\doi}[1]{doi: #1}\else
  \providecommand{\doi}{doi: \begingroup \urlstyle{rm}\Url}\fi

\bibitem[Agrawal et~al.(2016)Agrawal, Avadhanula, Goyal, and
  Zeevi]{agrawal2016near}
Shipra Agrawal, Vashist Avadhanula, Vineet Goyal, and Assaf Zeevi.
\newblock A near-optimal exploration-exploitation approach for assortment
  selection.
\newblock In \emph{Proceedings of the 2016 ACM Conference on Economics and
  Computation}, pages 599--600, 2016.

\bibitem[Agrawal et~al.(2017)Agrawal, Avadhanula, Goyal, and
  Zeevi]{agrawal2017thompson}
Shipra Agrawal, Vashist Avadhanula, Vineet Goyal, and Assaf Zeevi.
\newblock Thompson sampling for the mnl-bandit.
\newblock In \emph{Conference on Learning Theory}, pages 76--78, 2017.

\bibitem[Agrawal et~al.(2019)Agrawal, Avadhanula, Goyal, and
  Zeevi]{agrawal2019mnl}
Shipra Agrawal, Vashist Avadhanula, Vineet Goyal, and Assaf Zeevi.
\newblock Mnl-bandit: A dynamic learning approach to assortment selection.
\newblock \emph{Operations Research}, 67\penalty0 (5):\penalty0 1453--1485,
  2019.

\bibitem[Auer et~al.(2002)Auer, Cesa-Bianchi, and Fischer]{auer2002finite}
Peter Auer, Nicolo Cesa-Bianchi, and Paul Fischer.
\newblock Finite-time analysis of the multiarmed bandit problem.
\newblock \emph{Machine learning}, 47\penalty0 (2-3):\penalty0 235--256, 2002.

\bibitem[Bubeck et~al.(2012)Bubeck, Cesa-Bianchi, et~al.]{bubeck2012regret}
S{\'e}bastien Bubeck, Nicolo Cesa-Bianchi, et~al.
\newblock Regret analysis of stochastic and nonstochastic multi-armed bandit
  problems.
\newblock \emph{Foundations and Trends{\textregistered} in Machine Learning},
  5\penalty0 (1):\penalty0 1--122, 2012.

\bibitem[Bubeck et~al.(2013)Bubeck, Wang, and Viswanathan]{bubeck2013multiple}
S{\'e}ebastian Bubeck, Tengyao Wang, and Nitin Viswanathan.
\newblock Multiple identifications in multi-armed bandits.
\newblock In \emph{International Conference on Machine Learning}, pages
  258--265, 2013.

\bibitem[Chen et~al.(2017)Chen, Chen, Zhang, and Zhou]{chen2017adaptive}
Jiecao Chen, Xi~Chen, Qin Zhang, and Yuan Zhou.
\newblock Adaptive multiple-arm identification.
\newblock In \emph{Proceedings of the 34th International Conference on Machine
  Learning-Volume 70}, pages 722--730. JMLR.org, 2017.

\bibitem[Chen et~al.(2014)Chen, Lin, King, Lyu, and
  Chen]{chen2014combinatorial}
Shouyuan Chen, Tian Lin, Irwin King, Michael~R Lyu, and Wei Chen.
\newblock Combinatorial pure exploration of multi-armed bandits.
\newblock In \emph{Advances in Neural Information Processing Systems}, pages
  379--387, 2014.

\bibitem[Chen et~al.(2013)Chen, Wang, and Yuan]{chen2013combinatorial}
Wei Chen, Yajun Wang, and Yang Yuan.
\newblock Combinatorial multi-armed bandit: General framework and applications.
\newblock In \emph{International Conference on Machine Learning}, pages
  151--159, 2013.

\bibitem[Chen et~al.(2016{\natexlab{a}})Chen, Hu, Li, Li, Liu, and
  Lu]{chen2016generalreward}
Wei Chen, Wei Hu, Fu~Li, Jian Li, Yu~Liu, and Pinyan Lu.
\newblock Combinatorial multi-armed bandit with general reward functions.
\newblock In \emph{Advances in Neural Information Processing Systems}, pages
  1659--1667, 2016{\natexlab{a}}.

\bibitem[Chen et~al.(2016{\natexlab{b}})Chen, Wang, Yuan, and
  Wang]{chen2016combinatorial}
Wei Chen, Yajun Wang, Yang Yuan, and Qinshi Wang.
\newblock Combinatorial multi-armed bandit and its extension to
  probabilistically triggered arms.
\newblock \emph{The Journal of Machine Learning Research}, 17\penalty0
  (1):\penalty0 1746--1778, 2016{\natexlab{b}}.

\bibitem[Chen and Wang(2018)]{chen2018note}
Xi~Chen and Yining Wang.
\newblock A note on a tight lower bound for capacitated mnl-bandit assortment
  selection models.
\newblock \emph{Operations Research Letters}, 46\penalty0 (5):\penalty0
  534--537, 2018.

\bibitem[Chen et~al.(2018)Chen, Li, and Mao]{chen2018nearly}
Xi~Chen, Yuanzhi Li, and Jieming Mao.
\newblock A nearly instance optimal algorithm for top-k ranking under the
  multinomial logit model.
\newblock In \emph{Proceedings of the Twenty-Ninth Annual ACM-SIAM Symposium on
  Discrete Algorithms}, pages 2504--2522. SIAM, 2018.

\bibitem[Dinkelbach(1967)]{dinkelbach1967nonlinear}
Werner Dinkelbach.
\newblock On nonlinear fractional programming.
\newblock \emph{Management science}, 13\penalty0 (7):\penalty0 492--498, 1967.

\bibitem[Hoeffding(1963)]{hoeffding1963probability}
Wassily Hoeffding.
\newblock Probability inequalities for sums of bounded random variables.
\newblock \emph{Journal of the American Statistical Association}, 58\penalty0
  (301):\penalty0 13--30, 1963.

\bibitem[Jamieson and Nowak(2014)]{jamieson2014best}
Kevin Jamieson and Robert Nowak.
\newblock Best-arm identification algorithms for multi-armed bandits in the
  fixed confidence setting.
\newblock In \emph{2014 48th Annual Conference on Information Sciences and
  Systems (CISS)}, pages 1--6. IEEE, 2014.

\bibitem[Janson(2018)]{janson2018tail}
Svante Janson.
\newblock Tail bounds for sums of geometric and exponential variables.
\newblock \emph{Statistics \& Probability Letters}, 135:\penalty0 1--6, 2018.

\bibitem[Jin et~al.(2019)Jin, Li, Wang, and Zhou]{jin2019asymptotically}
Yaonan Jin, Yingkai Li, Yining Wang, and Yuan Zhou.
\newblock On asymptotically tight tail bounds for sums of geometric and
  exponential random variables.
\newblock \emph{arXiv preprint arXiv:1902.02852}, 2019.

\bibitem[Karnin et~al.(2013)Karnin, Koren, and Somekh]{karnin2013almost}
Zohar Karnin, Tomer Koren, and Oren Somekh.
\newblock Almost optimal exploration in multi-armed bandits.
\newblock In \emph{International Conference on Machine Learning}, pages
  1238--1246, 2013.

\bibitem[K{\"o}k and Fisher(2007)]{kok2007demand}
A~G{\"u}rhan K{\"o}k and Marshall~L Fisher.
\newblock Demand estimation and assortment optimization under substitution:
  Methodology and application.
\newblock \emph{Operations Research}, 55\penalty0 (6):\penalty0 1001--1021,
  2007.

\bibitem[Lai and Robbins(1985)]{lai1985asymptotically}
Tze~Leung Lai and Herbert Robbins.
\newblock Asymptotically efficient adaptive allocation rules.
\newblock \emph{Advances in applied mathematics}, 6\penalty0 (1):\penalty0
  4--22, 1985.

\bibitem[Lattimore and Szepesvári(2020)]{lattimore2020}
Tor Lattimore and Csaba Szepesvári.
\newblock \emph{Bandit Algorithms}.
\newblock Cambridge University Press, 2020.

\bibitem[Luce(2012)]{luce2012individual}
R~Duncan Luce.
\newblock \emph{Individual choice behavior: A theoretical analysis}.
\newblock Courier Corporation, 2012.

\bibitem[MCFADDEN(1973)]{mcfadden1973conditional}
D~MCFADDEN.
\newblock Conditional logit analysis of qualitative choice behavior.
\newblock \emph{Frontiers in Econometrics}, pages 105--142, 1973.

\bibitem[Rejwan and Mansour(2020)]{rejwan2020top}
Idan Rejwan and Yishay Mansour.
\newblock Top-$k$ combinatorial bandits with full-bandit feedback.
\newblock In \emph{Algorithmic Learning Theory}, pages 752--776, 2020.

\bibitem[Rusmevichientong et~al.(2010)Rusmevichientong, Shen, and
  Shmoys]{rusmevichientong2010dynamic}
Paat Rusmevichientong, Zuo-Jun~Max Shen, and David~B Shmoys.
\newblock Dynamic assortment optimization with a multinomial logit choice model
  and capacity constraint.
\newblock \emph{Operations research}, 58\penalty0 (6):\penalty0 1666--1680,
  2010.

\bibitem[Saha and Gopalan(2019)]{saha2019combinatorial}
Aadirupa Saha and Aditya Gopalan.
\newblock Combinatorial bandits with relative feedback.
\newblock In \emph{Advances in Neural Information Processing Systems}, pages
  983--993, 2019.

\bibitem[Saur{\'e} and Zeevi(2013)]{saure2013optimal}
Denis Saur{\'e} and Assaf Zeevi.
\newblock Optimal dynamic assortment planning with demand learning.
\newblock \emph{Manufacturing \& Service Operations Management}, 15\penalty0
  (3):\penalty0 387--404, 2013.

\bibitem[Slivkins et~al.(2019)]{slivkins2019introduction}
Aleksandrs Slivkins et~al.
\newblock Introduction to multi-armed bandits.
\newblock \emph{Foundations and Trends{\textregistered} in Machine Learning},
  12\penalty0 (1-2):\penalty0 1--286, 2019.

\bibitem[Soufiani et~al.(2013)Soufiani, Parkes, and
  Xia]{soufiani2013preference}
Hossein~Azari Soufiani, David~C Parkes, and Lirong Xia.
\newblock Preference elicitation for general random utility models.
\newblock In \emph{Proceedings of the Twenty-Ninth Conference on Uncertainty in
  Artificial Intelligence}, pages 596--605, 2013.

\bibitem[Train(2009)]{train2009discrete}
Kenneth~E Train.
\newblock \emph{Discrete choice methods with simulation}.
\newblock Cambridge university press, 2009.

\bibitem[Wang et~al.(2018)Wang, Chen, and Zhou]{wang2018near}
Yining Wang, Xi~Chen, and Yuan Zhou.
\newblock Near-optimal policies for dynamic multinomial logit assortment
  selection models.
\newblock In \emph{Advances in Neural Information Processing Systems}, pages
  3101--3110, 2018.

\end{thebibliography}
\end{document}